\documentclass{article}

\usepackage{iclr2027_conference,times}

\usepackage[utf8]{inputenc}
\usepackage[T1]{fontenc}
\usepackage{hyperref}
\hypersetup{
  pdftitle={STEPS: Selective On-Policy Self-Distillation for Reasoning},
  pdfauthor={},
  pdfsubject={Preprint},
  pdfkeywords={on-policy self-distillation, reasoning, GRPO, span selection},
}

\usepackage{url}
\usepackage{booktabs}
\usepackage{array}
\usepackage{amsfonts}
\usepackage{amsmath}
\usepackage{amssymb}
\usepackage{amsthm}
\usepackage{nicefrac}
\usepackage{microtype}
\usepackage{xcolor}
\usepackage{colortbl}
\usepackage{graphicx}
\usepackage{multirow}
\usepackage{algorithm}
\usepackage{algpseudocode}
\usepackage{enumitem}
\usepackage{subcaption}
\usepackage{wrapfig}
\usepackage{placeins}      
\usepackage{float}         
\usepackage[framemethod=default]{mdframed}
\usepackage{xspace}

\definecolor{promptInstrTitle}{HTML}{B8853E}   
\definecolor{promptInstrBg}{HTML}{FBF1DE}
\definecolor{promptKeyTitle}{HTML}{6B9A4F}     
\definecolor{promptKeyBg}{HTML}{E8F2DC}
\definecolor{promptRubricTitle}{HTML}{BB5A5A}  
\definecolor{promptRubricBg}{HTML}{F8E0E0}
\definecolor{promptInputTitle}{HTML}{4E89C5}   
\definecolor{promptInputBg}{HTML}{DEEBF6}
\definecolor{schemaLabel}{HTML}{555555}        
\definecolor{traceRow}{HTML}{EDF3F8}           
\definecolor{onlineSelfRow}{HTML}{F0EBF8}      
\definecolor{degradationAvg}{HTML}{FBE7E7}    
\definecolor{citationPurple}{HTML}{6F42A1}
\definecolor{referenceBlue}{HTML}{1F4E79}
\hypersetup{
  colorlinks=true,
  citecolor=citationPurple,
  linkcolor=referenceBlue,
  urlcolor=referenceBlue,
  pdfborder={0 0 0},
}

\DeclareRobustCommand{\paperref}[1]{\textcolor{referenceBlue}{#1}}

\newsavebox{\stepsTableMeasure}
\newcount\stepsTableGaps
\NewDocumentEnvironment{fullwidthtablebody}{m m +b}{%
  \sbox{\stepsTableMeasure}{\begin{tabular}{#1}#3\end{tabular}}%
  \stepsTableGaps=\numexpr2*#2\relax
  \setlength{\tabcolsep}{\dimexpr\tabcolsep+(\linewidth-\wd\stepsTableMeasure)/\stepsTableGaps\relax}%
  \begin{tabular}{#1}#3\end{tabular}%
}{}

\newenvironment{promptblock}[3]{%
  \begin{mdframed}[
    frametitle={\strut #1\strut},
    frametitlebackgroundcolor=#2,
    frametitlefont={\centering\bfseries\sffamily\small\color{white}},
    frametitlerule=false,
    frametitlealignment=\centering,
    frametitleaboveskip=6pt,
    frametitlebelowskip=6pt,
    backgroundcolor=#3,
    hidealllines=true,
    roundcorner=4pt,
    innerleftmargin=11pt,
    innerrightmargin=11pt,
    innertopmargin=8pt,
    innerbottommargin=8pt,
    skipabove=10pt,
    skipbelow=6pt,
  ]%
  \small\sffamily\setlength{\parskip}{4pt}\raggedright\sloppy
}{\end{mdframed}}

\newcommand{\sk}[1]{\textcolor{schemaLabel}{\texttt{#1}}}

\AddToHook{begindocument/end}{%
  \renewcommand{\textit}[1]{{\upshape #1}}%
  \renewcommand{\textsl}[1]{{\upshape #1}}%
  \renewcommand{\emph}[1]{{\upshape #1}}%
}

\theoremstyle{definition}
\newtheorem{theorem}{Theorem}
\newtheorem{lemma}[theorem]{Lemma}
\newtheorem{proposition}[theorem]{Proposition}
\newtheorem{corollary}[theorem]{Corollary}

\newtheorem{remark}[theorem]{Remark}

\newcommand{\method}{\textsc{STEPS}\xspace}

\title{\centering STEPS: Selective On-Policy Self-Distillation for Reasoning}
\author{%
  \parbox{\dimexpr\textwidth-2\tabcolsep\relax}{\centering
    \normalsize\bfseries
    Jiaxuan Wang\textsuperscript{1,2,3}\thanks{Work done while Jiaxuan Wang (\href{mailto:wangjx@lamda.nju.edu.cn}{wangjx@lamda.nju.edu.cn}) was an intern at AMAP, Alibaba.}\quad
    Xuan Ouyang\textsuperscript{4}\quad
    Zhiyu Chen\textsuperscript{5}\\[2pt]
    Yulan Hu\textsuperscript{3,\ensuremath{\dagger}}\quad
    Lan-Zhe Guo\textsuperscript{1,2}\thanks{Corresponding authors: Yulan Hu and Lan-Zhe Guo (\href{mailto:guolz@nju.edu.cn}{guolz@nju.edu.cn}).}\\[6pt]
    \small\normalfont
    \textsuperscript{1}State Key Laboratory of Novel Software Technology, Nanjing University\\
    \textsuperscript{2}School of Intelligence Science and Technology, Nanjing University\\
    \textsuperscript{3}AMAP, Alibaba\quad
    \textsuperscript{4}University of Wisconsin--Madison\\
    \textsuperscript{5}Tsinghua University
  }%
}
\hypersetup{pdfauthor={Jiaxuan Wang, Xuan Ouyang, Zhiyu Chen, Yulan Hu, Lan-Zhe Guo}}

\iclrfinalcopy

\begin{document}

\maketitle
\raggedbottom
\fancyhead{}
\renewcommand{\headrulewidth}{0pt}

\begin{abstract}
On-policy self-distillation (OPSD) uses a model as its own teacher under privileged context, providing token-level supervision on the model's own reasoning trajectories. However, persistent all-token guidance can degrade training in our math RL setting. Such guidance may unnecessarily constrain non-critical tokens and reinforce biases induced by privileged information unavailable at inference. Motivated by these risks, we propose STEPS, a selective OPSD framework that controls the location, direction, and duration of distillation. STEPS identifies critical spans in student rollouts, applies forward KL to key spans or reverse KL to error spans, and gradually transitions to pure GRPO after a short distillation phase. Across four models from three families, STEPS improves mean accuracy over GRPO. On Qwen3-8B, it raises average accuracy across four math benchmarks and GPQA-Diamond by an absolute 2.76\% with a strong annotator. Using the training policy as its own annotator yields an absolute mean gain of 1.90\%, with an estimated 2.7\% training-time overhead amortized over 300 steps. Controlled ablations support semantic span selection, class-specific KL directions, and the hand-off to GRPO. Together, these results support brief, targeted self-distillation as an efficient complement to reward-driven reasoning training.

\end{abstract}

\section{Introduction}
\label{sec:intro}

Reinforcement learning with verifiable rewards (RLVR), exemplified by GRPO~\citep{shao2024deepseekmath}, improves reasoning through outcome-level feedback~\citep{guo2025deepseek,olmo2025}. However, a trajectory-level reward does not directly identify which intermediate decisions require correction~\citep{lightman2024prm,kazemnejad2025vineppo}. On-policy self-distillation (OPSD) offers a complementary source of supervision: the model itself, or a policy snapshot, acts as a teacher conditioned on privileged context such as verified solutions or environment feedback~\citep{huebotter2026sdpo,zhao2026opsd}. Its predictions provide token-level guidance along the student's own rollouts, making fine-grained self-supervision available without a separate teacher model~\citep{shenfeld2026sdft}.

\begin{figure}[!t]
  \centering
  \includegraphics[width=\linewidth]{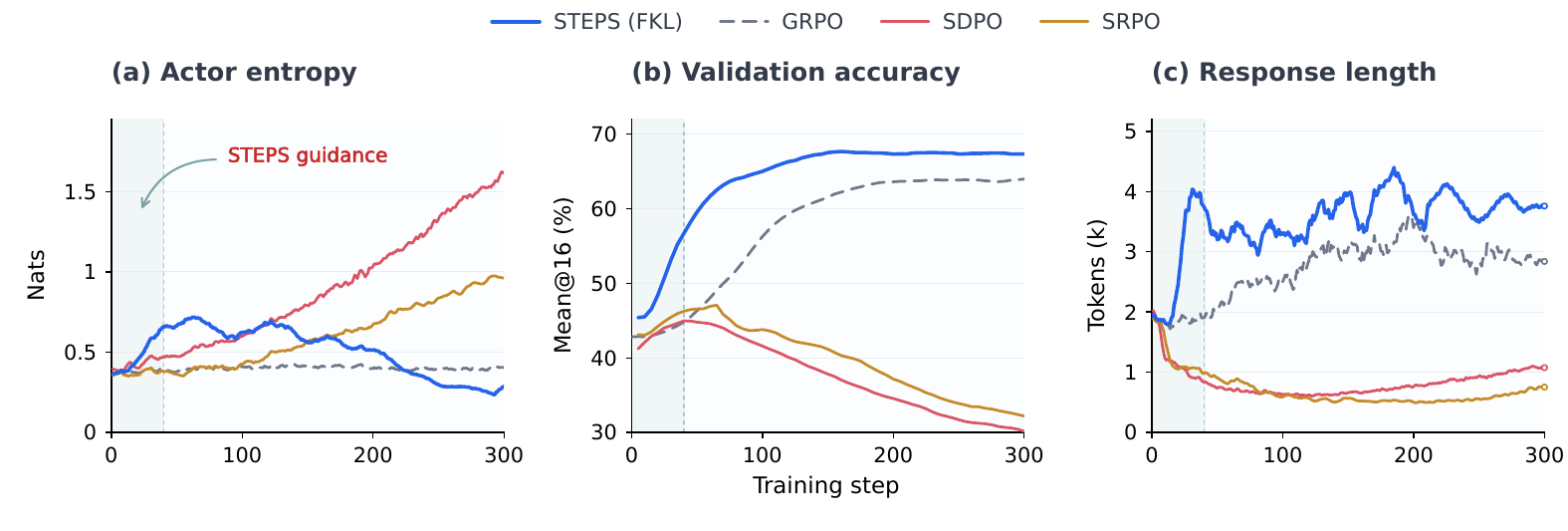}
  \caption{\textbf{Training diagnostics on Qwen3-8B.} (a) Per-token actor entropy, (b) OpenThoughts held-out validation accuracy (mean@16), and (c) response length. In these runs, SDPO and SRPO exhibit rising entropy and shorter responses alongside validation degradation, while STEPS (FKL on key spans) sustains its validation gains. Shading marks the initial STEPS guidance phase.}
  \label{fig:teaser}
\end{figure}

However, a teacher's prediction under privileged context is useful only insofar as it improves the student's ability to solve problems without that context~\citep{penaloza2026piDistill,yang2026rlsd}. Applying this guidance uniformly across tokens and maintaining it throughout training may impose constraints beyond those needed for useful correction~\citep{kim2026why,xx2026srpo}. In our reproduced SDPO and SRPO runs on Qwen3-8B math RL, responses shorten, per-token actor entropy rises, and validation accuracy falls by an absolute 18--19\% from early peaks, while GRPO remains stable under the same data and optimizer stack, as shown in \paperref{Figure~\ref{fig:teaser}} and \paperref{Appendix~\ref{app:three-symptom}}. These observations raise a central question: how can OPSD retain useful corrective guidance while preserving the role of verifier-grounded learning?

The key idea is to use self-distillation as a brief correction phase focused on semantically important reasoning spans. Prior work on token importance shows that a subset of positions can retain much of the benefit of distillation~\citep{xx2026tip}, motivating selective guidance instead of uniform coverage. Limiting its duration gives the auxiliary privileged signal a defined role in training: support early correction before optimization returns fully to verifier rewards. This also motivates analyzing how much privileged-gradient variation the student encounters during the guided phase~\citep{yang2026rlsd}.

We introduce \textbf{\method (Span-Targeted Early-Phase Self-Distillation)}, a selective OPSD framework that controls the location, direction, and duration of guidance. An annotator identifies key spans in correct rollouts and error spans in incorrect rollouts. Its span masks determine where distillation is applied, while its coarse diagnostic labels provide privileged context for the self-teacher. We consider forward KL (FKL) on key spans and reverse KL (RKL) on error spans as alternative configurations. \method caps span coverage at 25\% and decays the KL coefficient to zero by step 40, while restoring GRPO on selected spans. Training then proceeds with pure GRPO, without annotation or teacher forward passes. The annotator can also be the training policy itself.

Our analysis motivates these choices through local gradient behavior and conditional exposure to privileged guidance. For fixed teacher targets, FKL provides non-vanishing logit pressure on teacher-supported tokens as their student probability approaches zero, whereas RKL's pressure is scaled by student probability and depends on the centered student--teacher log-ratio, as established by Corollary~\ref{cor:fkl-underalloc} in \paperref{Appendix~\ref{app:proof-softmax}}. These identities motivate testing KL directions by span class. For the FKL branch, Proposition~\ref{prop:exposure} in \paperref{Appendix~\ref{app:proof-exposure}} characterizes how restricted coverage and finite-duration decay limit cumulative privileged-gradient variation. Proposition~\ref{prop:alignment} in \paperref{Appendix~\ref{app:proof-alignment}} gives conditions under which selected key spans retain a positive learning signal.

Experiments on four models from three families support this selective use of self-distillation. On Qwen3-8B, \method with a strong annotator improves the average accuracy across four math benchmarks and GPQA-Diamond by an absolute 2.76\% over GRPO. Online self-annotation yields an absolute mean gain of 1.90\%, with an estimated 2.7\% training-time overhead amortized over 300 steps. FKL on key spans performs better on Qwen3-8B, while RKL on error spans performs better on Qwen3-1.7B, demonstrating an empirical dependence on the tested model. Additional Llama and SmolLM3 results extend the evaluation across families. Controlled mask, class-by-direction, and hand-off ablations probe these design choices through matched comparisons, with the hand-off experiment evaluating the coupled decay of KL and restoration of GRPO.

Our contributions are threefold:
\begin{itemize}[leftmargin=1.1em,itemsep=1pt,topsep=2pt]
\item \textbf{Method.} We introduce \method, a selective OPSD framework that combines semantic span selection, class-specific KL directions, and a brief distillation phase followed by pure GRPO.
\item \textbf{Theoretical analysis.} We characterize local KL-gradient asymmetry and derive conditional bounds on privileged-gradient exposure and retained key-span signal.
\item \textbf{Empirical evidence.} Across four models from three families, \method improves mean accuracy over GRPO. It also demonstrates higher training-step efficiency: SmolLM3 reaches 71.84 AVG at step 40, exceeding GRPO's 69.45 at step 325. Matched ablations support the design choices, while online self-annotation retains gains at low estimated amortized cost.
\end{itemize}

\begingroup
\raggedbottom
\setlength{\parskip}{5pt}
\renewcommand{\paragraph}[1]{\par\noindent\textbf{#1}\quad}

\section{Related Work}
\label{sec:related}

\paragraph{On-Policy Distillation.}
On-policy distillation (OPD) extends knowledge distillation to the student's generation distribution by providing teacher feedback on student-sampled trajectories~\citep{hinton2015distilling,gu2024minillm,agarwal2024gkd,lu2025tmlopd}. Research on divergence design, sample reuse, and prompt selection refines how this feedback transfers between teachers and students~\citep{ko2024distillm,wang2026promptcount}. Beyond model compression, OPD provides a route to capability consolidation alongside parameter-space model merging~\citep{wu2026fusion}. Industrial applications span strong-to-weak transfer in Alibaba's Qwen3~\citep{qwen3report}, domain-expert consolidation in Xiaomi's MiMo-V2-Flash, and cross-stage capability recovery in Z.ai's GLM-5~\citep{xiaomi2026mimov2flash,glm2026glm5}. These applications transfer or consolidate capabilities from trained teachers into a deployable student. \method studies a complementary setting: privileged predictions from the same policy provide localized, temporary supervision during reward-driven reasoning training.

\paragraph{Privileged Self-Distillation and Training Stability.}
Privileged self-distillation obtains dense supervision by conditioning the policy or a snapshot on context available only during training~\citep{huebotter2026sdpo,zhao2026opsd}. The reliability of auxiliary guidance and the stability of optimization are recurring concerns in self-distillation and RLVR~\citep{kim2026why,yang2026rlsd,cai2026steer}. Objective routing and advantage reweighting address these concerns by adapting updates to response correctness or teacher--student differences~\citep{xx2026srpo,yang2026rlsd}. These approaches regulate how teacher information enters the policy update. \method separates teacher conditioning from loss localization: coarse diagnostic labels supply privileged context, while semantic masks determine where KL guidance acts and span classes determine which direction is used.

\paragraph{Token-Level Guidance and Credit Assignment.}
Fine-grained supervision aims to concentrate learning on informative reasoning decisions. Token-importance analyses motivate selective distillation~\citep{xx2026rethinkingopd,xx2026tip,lin2025critical}, with hard selection and soft weighting using signals such as student uncertainty, teacher confidence, and distributional disagreement~\citep{xx2026tip,li2026fireopd}. Complementary approaches refine credit through process rewards~\citep{lightman2024prm,wang2024mathshepherd}, step-level value estimates~\citep{kazemnejad2025vineppo}, or counterfactual redistribution of response-level advantages~\citep{zhang2026cort}; query-specific rubrics instead refine the criteria used to evaluate outputs~\citep{xx2026rubricarm}. These methods improve the granularity or reliability of supervision through different signals. \method directly localizes an auxiliary KL loss with semantic span masks and compares KL directions within each span class, without requiring a separately trained process reward model.

\paragraph{Combining Distillation with Reinforcement Learning.}
Hybrid post-training combines the outcome-level feedback of RLVR~\citep{shao2024deepseekmath,yu2025dapo,olmo2025} with dense teacher supervision. Methods such as SDPG and ATOD regulate this combination through response-level gating, turn-level reweighting, and scheduled distillation coefficients~\citep{liu2026sdpg,tan2026atod}. Their shared aim is to exploit useful teacher feedback while retaining reward-driven improvement. \method studies this balance at the semantic-span level: it couples KL decay with the restoration of GRPO on selected positions and ends self-teacher guidance after a finite initial phase. This makes the support, direction, and duration of the auxiliary objective explicit. \paperref{Table~\ref{tab:paradigms}} in \paperref{Appendix~\ref{app:paradigms}} summarizes the training signals and teacher information used by the compared paradigms.

\par\newpage
\endgroup

\section{STEPS: Selective On-Policy Self-Distillation}
\label{sec:method}

\method uses self-distillation to guide selected reasoning decisions during an initial phase of RLVR, then continues with pure GRPO. The method separates three choices: which positions receive teacher guidance, which KL direction is applied, and how guidance is withdrawn as reward-driven learning takes over. \paperref{Figure~\ref{fig:pipeline}} summarizes the training pipeline.

\begin{figure}[t]
  \centering
  \includegraphics[width=\linewidth]{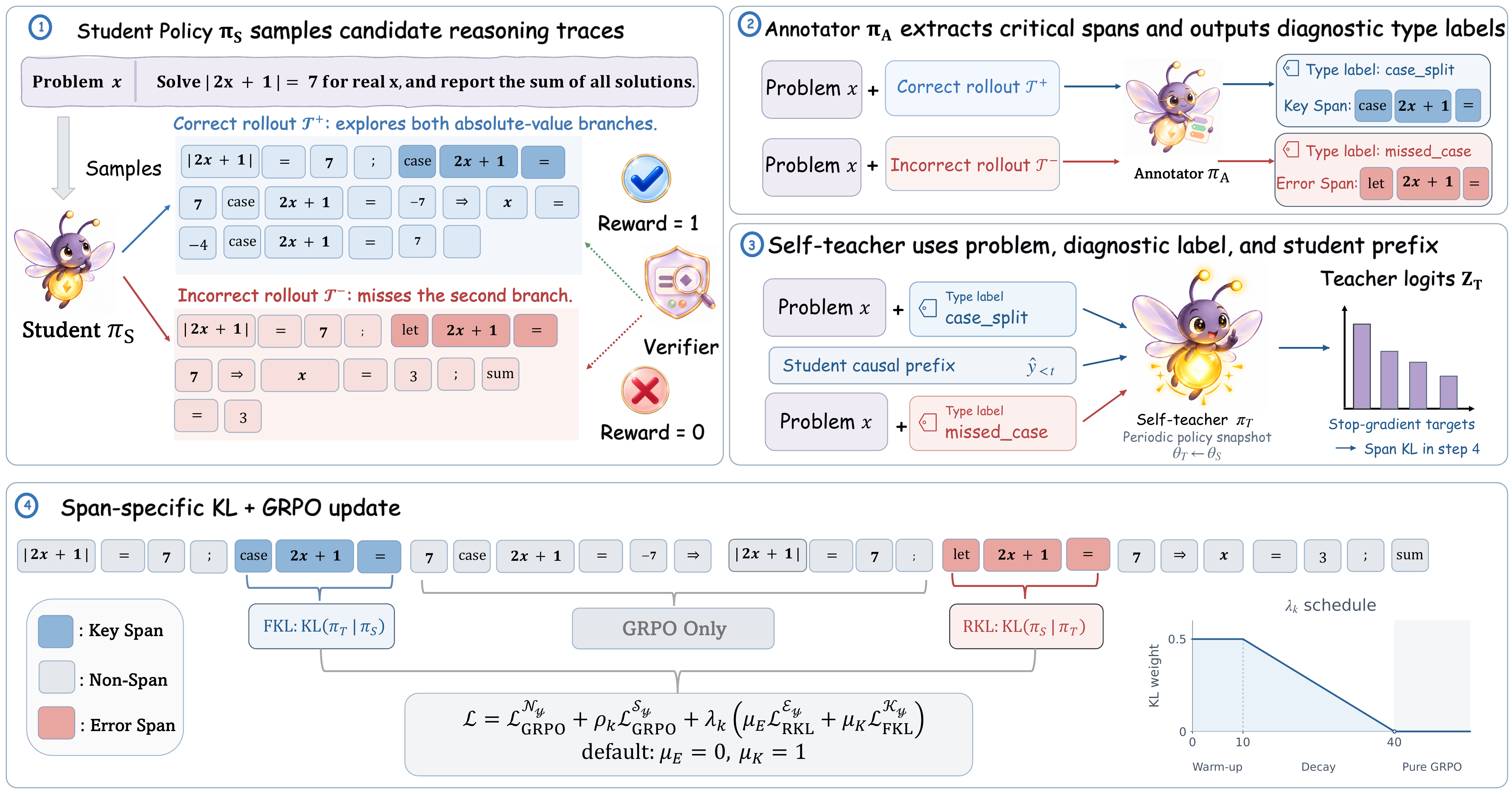}
  \caption{\textbf{\method pipeline.} The student samples a rollout and receives a verifier outcome. An annotator supplies span masks and a coarse diagnostic label; the self-teacher conditions on this label and the student's causal prefix. Span-specific KL is combined with GRPO, which is restored on selected spans as KL decays to zero.}
  \label{fig:pipeline}
\end{figure}

\subsection{Span Annotation and Self-Teacher Construction}
\label{ssec:span-annot}

\paragraph{Locating reasoning spans.} The verifier indicates whether a rollout succeeds, but does not locate the reasoning steps to reinforce or correct. Given a problem \(x\), the student policy \(\pi_S=\pi_\theta\) samples a rollout \(\hat y\), which receives a verifier outcome \(R(x,\hat y)\in\{0,1\}\). During the distillation phase, we split the response into numbered segments and ask an annotator \(\pi_A\) to identify key reasoning spans in correct rollouts or localized error spans in incorrect rollouts. Segment indices are projected to token positions, yielding key and error sets \(\mathcal K_y\) and \(\mathcal E_y\). Let \(\mathcal S_y=\mathcal K_y\cup\mathcal E_y\) denote all annotated positions and \(\mathcal N_y\) the remaining response positions. We cap coverage at \(|\mathcal S_y|\leq\alpha|\hat y|\), with \(\alpha=0.25\). The annotator may be an external model or the current student policy; its outputs are discrete annotations and receive no gradient.

\paragraph{Turning annotations into teacher guidance.} Each annotation provides both span locations and a coarse diagnostic label, such as \texttt{case\_split} or \texttt{missed\_case}. These outputs have separate roles: the locations determine where distillation is applied, while the label forms a private diagnostic prefix \(c\) for the self-teacher. The teacher parameters \(\theta_T\) are copied from the student and periodically synchronized during the distillation phase. At position \(t\), the student and teacher distributions are

\[
\pi_{S,t}=\pi_\theta\bigl(\cdot\mid x,\hat y_{<t}\bigr),
\qquad
\pi_{T,t}=\pi_{\theta_T}\bigl(\cdot\mid x,c,\hat y_{<t}\bigr).
\]

Both distributions are evaluated on the student's causal prefix. Only the coarse diagnostic description is added as privileged context; span locations are used in the loss mask. Teacher logits and span masks are treated as stop-gradient targets. \paperref{Appendices~\ref{app:prompts}--\ref{app:span-token-alignment}} provide the prompts, annotation schema, and segment-to-token alignment.

\subsection{Span-Specific Distillation}
\label{ssec:loss}

Once the spans and teacher distributions are available, \method specifies the KL objective on each span class. We consider forward KL (FKL) on key spans and reverse KL (RKL) on error spans. FKL can lift teacher-supported alternatives even when their student probabilities are small, whereas RKL weights its centered log-ratio correction by the student's current probability. These distinct local effects motivate evaluating the directions by span class; Section~\ref{sec:theory} formalizes the distinction. Binary switches \(\mu_E,\mu_K\in\{0,1\}\) enable the corresponding branches and are fixed for each training run. The default Qwen3-8B configuration is \((\mu_E,\mu_K)=(0,1)\); \((1,0)\) enables the RKL-on-error alternative. The empirical comparison is reported in Section~\ref{sec:exp}.

For a position set \(\mathcal A\), let \(\mathcal L_{\mathrm{GRPO}}^{\mathcal A}\) denote the GRPO token loss summed over \(\mathcal A\) and normalized by the full rollout length \(|\hat y|\). The per-rollout objective at training step \(k\) is

\begin{equation}
\begin{aligned}
\mathcal L^{(k)}(\hat y;\theta)
&=\mathcal L_{\mathrm{GRPO}}^{\mathcal N_y}
  +\rho_k\,\mathcal L_{\mathrm{GRPO}}^{\mathcal S_y}\\[3pt]
&\quad+\frac{\lambda_k}{|\hat y|}\Biggl[
  \mu_E\sum_{t\in\mathcal E_y}
  D_{\mathrm{KL}}\!\bigl(\pi_{S,t}\,\|\,\operatorname{sg}(\pi_{T,t})\bigr)
  +\mu_K\sum_{t\in\mathcal K_y}
  D_{\mathrm{KL}}\!\bigl(\operatorname{sg}(\pi_{T,t})\,\|\,\pi_{S,t}\bigr)
  \Biggr].
\end{aligned}
\label{eq:loss}
\end{equation}

Here \(\operatorname{sg}\) denotes stop-gradient, \(\lambda_k\) controls distillation strength, and \(\rho_k\) controls the restoration of GRPO on annotated positions. Non-span positions retain the GRPO objective throughout training. The GRPO weight on an annotated span is governed by \(\rho_k\), independently of whether its KL branch is enabled. All loss components use the same sequence normalization, and empty sets contribute zero. The implementation clips per-vocabulary KL contributions at \(\tau=0.05\) before summation \citep{zhao2026opsd}, as detailed in \paperref{Appendix~\ref{app:hparams-train}}; Equation~\eqref{eq:loss} displays the underlying divergences analyzed in Section~\ref{sec:theory}.

\subsection{Distillation Decay and GRPO Handoff}
\label{ssec:decay}

The span mask localizes teacher guidance, while the schedule limits how long that guidance remains active. We hold the KL coefficient at \(w_0\) for an initial interval, then decrease it linearly:
\nopagebreak[4]
\begin{equation}
\lambda_k=w_0
\begin{cases}
1, & k<t_{\mathrm{start}},\\[3pt]
1-\dfrac{k-t_{\mathrm{start}}}{T_{\mathrm{decay}}},
&t_{\mathrm{start}}\leq k\leq t_{\mathrm{start}}+T_{\mathrm{decay}},\\[5pt]
0, & k>t_{\mathrm{start}}+T_{\mathrm{decay}}.
\end{cases}
\label{eq:lambda-schedule}
\end{equation}

We couple this decay to \(\rho_k=1-\lambda_k/w_0\), so that GRPO is gradually restored on the same annotated positions as teacher guidance is withdrawn. Our configuration uses \(w_0=0.5\), \(t_{\mathrm{start}}=10\), and \(T_{\mathrm{decay}}=30\). At step 40, \(\lambda_k=0\) and \(\rho_k=1\), and Equation~\eqref{eq:loss} reduces to GRPO over all response positions. Annotation and teacher forward passes are then skipped for the remainder of training.

During the active phase, the teacher is synchronized from the student every \(N=10\) steps. Algorithm~\ref{alg:trace} gives the update procedure; training and evaluation protocols are detailed in \paperref{Appendix~\ref{app:hparams}}. Together, the mask and schedule confine privileged guidance to selected positions within a finite training window. Section~\ref{sec:theory} analyzes the local KL effects and the trade-off between limiting privileged guidance and retaining useful training signals.

\section{Theoretical Analysis}
\label{sec:theory}

We analyze the unclipped losses with stop-gradient teacher targets and masks. The KL identities explain the local role of divergence direction; the exposure and signal bounds characterize selective guidance for the FKL-on-key branch. Full statements and proofs are deferred to \paperref{Appendix~\ref{app:proofs}}.

\subsection{KL Gradient Asymmetry}
\label{ssec:thy-mechanism}
\label{ssec:thy-setup}

At a fixed rollout position, let \(p_v=\pi_S(v)>0\), \(q_v=\pi_T(v)>0\), \(r_v=\log(p_v/q_v)\), and \(\bar r=\sum_vp_vr_v\). For student logits \(\ell_v\), the two KL directions yield
\begin{equation}
\frac{\partial D_{\mathrm{KL}}(q\|p)}{\partial\ell_v}=p_v-q_v,
\qquad
\frac{\partial D_{\mathrm{KL}}(p\|q)}{\partial\ell_v}=p_v(r_v-\bar r).
\label{eq:kl-grads}
\end{equation}

For fixed \(q_v>0\) and bounded \(|\bar r|\), as \(p_v\to0\), FKL retains upward logit pressure approaching \(q_v\), while RKL pressure vanishes. RKL instead weights its centered log-ratio correction by student probability, with downward pressure exactly when \(r_v>\bar r\). \paperref{Figure~\ref{fig:kl-gradient-asymmetry}} illustrates these local differences, motivating comparisons of KL direction by span class; the preferred configuration is determined empirically.

\begin{figure}[!t]
\centering
\includegraphics[width=\linewidth]{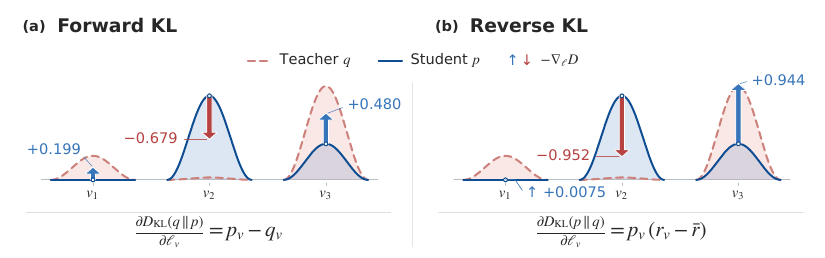}
\caption{\textbf{KL gradient asymmetry.} Smooth peaks illustrate discrete token probabilities $p=(0.001,0.699,0.300)$ and $q=(0.200,0.020,0.780)$. The six arrows show negative logit gradients of unclipped KL with fixed teacher $q$, on a shared scale independent of probability heights; they do not represent probability displacements. Here $r_v=\log(p_v/q_v)$ and $\bar r=\sum_v p_v r_v$.}
\label{fig:kl-gradient-asymmetry}
\end{figure}

\begin{table}[!b]
\centering\small
\setlength{\tabcolsep}{4pt}
\caption{Post-update key-token log-probability lift.}
\label{tab:lift-main}
\begin{fullwidthtablebody}{l >{\columncolor{traceRow}}c >{\columncolor{traceRow}}c cccc}{7}
\toprule
Method & \method-FKL & \method-RKL & RLSD & GRPO & SRPO & SDPO \\
\midrule
$\Delta_{\mathrm{lift}}$ (nats) & \textbf{0.145} & 0.078 & 0.068 & 0.054 & 0.027 & 0.019 \\
\bottomrule
\end{fullwidthtablebody}
\end{table}

Consistent with the local mechanism above, \method-FKL yields the largest post-update key-token log-probability lift on the shared held-out token set in \paperref{Table~\ref{tab:lift-main}}. \paperref{Appendix~\ref{app:case-study}} provides the metric definition and detailed measurement protocol.
\paperref{Figure~\ref{fig:token-update-localization}} complements this aggregate measure with token-level comparisons against GRPO.

\subsection{Privileged-Gradient Exposure}
\label{ssec:thy-bound}

To analyze the extent and duration of guidance, let \(\mathcal F_k\) contain the training history, problem, rollout, and selected mask. Let \(\delta_{k,t}\) be the per-token FKL gradient centered over privileged context conditional on \(\mathcal F_k\)~\citep{yang2026rlsd}. Under the uniform score-operator bound \(C_s\), coverage cap \(\alpha\), and a bound \(\bar V\) on the expected average conditional teacher variance over selected positions, \paperref{Appendix~\ref{app:proof-exposure}} establishes
\begin{equation}
\mathcal E_K:=\sum_{k=1}^{K}\lambda_k^2\,
\mathbb E\!\left[
\frac{1}{|\hat y_k|}\sum_{t\in\mathcal K_{y_k}}\|\delta_{k,t}\|_2^2
\right]
\leq C_s^2\alpha\bar V\sum_{k=1}^{K}\lambda_k^2.
\label{eq:exposure-bound}
\end{equation}

The mask limits coverage through \(\alpha\), while decay limits cumulative weighting through \(\sum_k\lambda_k^2\). The finite schedule in Section~\ref{sec:method} makes this upper bound independent of the total training horizon. It controls centered conditional variation, not shifts in the conditional teacher mean or overall training stability.

\subsection{Key-Span Signal Bound}
\label{ssec:thy-alignment}

Restricting guidance is useful only if the selected positions retain informative updates. Conditional on \(\mathcal F_k\), suppose every nonempty selected key set has expected precision at least \(q_K\). Relative to a common \(\mathcal F_k\)-measurable oracle loss-gradient direction \(\tilde g_k\), assume expected per-token alignment, conditional on \(\mathcal F_k\) and critical status, is at least \(\gamma_K>0\) for true-critical positions and \(-B_K\), \(B_K\geq0\), for false positives. With \(p_K=\mathbb E[|\mathcal K_y|/|\hat y|]\) and the sequence-normalized FKL gradient \(g_k^{\mathrm{sel}}=\lambda_k|\hat y|^{-1}\sum_{t\in\mathcal K_y}g_{\mathrm{FKL}}(t)\), \paperref{Appendix~\ref{app:proof-alignment}} gives
\begin{equation}
\mathbb E\!\left[\left\langle g_k^{\mathrm{sel}},\tilde g_k\right\rangle\right]
\geq\lambda_kp_K\left[q_K\gamma_K-(1-q_K)B_K\right].
\label{eq:signal-bound}
\end{equation}

The bound makes selection quality explicit: selected spans retain positive expected alignment when true-critical contributions outweigh false-positive effects and \(\lambda_kp_K>0\). Together, the two bounds describe how limited guidance can retain useful signal under explicit selection and alignment assumptions.

\begingroup
\setlength{\textfloatsep}{12pt plus 2pt minus 2pt}
\setlength{\floatsep}{10pt plus 2pt minus 2pt}
\setlength{\intextsep}{10pt plus 2pt minus 2pt}
\captionsetup{font=small}

\section{Experiments}
\label{sec:exp}

\subsection{Experimental Setup}
\label{sec:exp:setup}

\paragraph{Model, data, evaluation.}
\label{para:strong-base}
We evaluate Qwen3-8B and Qwen3-1.7B~\citep{qwen3report}, Llama-3.2-3B-Instruct~\citep{meta2024llama32}, and SmolLM3-3B~\citep{bakouch2025smollm3}. Training uses the OpenThoughts-114k math subset~\citep{openthoughts} (up to 30K problem--solution pairs), verl~\citep{verl}, H100 GPUs, and DAPO clip-higher~\citep{yu2025dapo}. Evaluation covers four math benchmarks---MATH-500~\citep{lightman2024prm,hendrycks2021math}, AIME 2024/2025~\citep{aime2024,aime2025}, and AMC 2023~\citep{amc2023}---and GPQA-Diamond~\citep{rein2023gpqa} for out-of-distribution generalization. Thinking-mode \texttt{avg@k} averages per-problem accuracy over $k$ samples. Accuracy scores are expressed in \%; gains and their confidence intervals are absolute differences on this scale. Cross-family runs reuse the configuration; hyperparameters, decoding settings, and model-specific details are in \paperref{Apps.~\ref{app:hparams} and~\ref{app:cross-family}}.

\paragraph{Baselines and checkpoint selection.}
\label{para:repro}
Baselines share the training stack: GRPO~\citep{shao2024deepseekmath} with recent improvements~\citep{olmo2025,khatri2026art}, SDPO~\citep{huebotter2026sdpo}, SRPO~\citep{xx2026srpo}, and RLSD~\citep{yang2026rlsd}. Qwen3 baselines use the best learning-rate/batch-size grid setting. We select the checkpoint with the highest OpenThoughts validation mean@16 among the evaluated checkpoints of each run. Held-out benchmark scores are never used for checkpoint selection. Coefficients and grids are in \paperref{Apps.~\ref{app:hparams-train} and~\ref{app:sdpo-sweep}}; degradation diagnostics appear in \paperref{Fig.~\ref{fig:teaser}} and \paperref{App.~\ref{app:three-symptom}}.

For Qwen models, the main \method, SDPO, and SRPO runs enable thinking for both the student and teacher during training. All methods in \paperref{Table~\ref{tab:main}} use thinking-enabled evaluation. \paperref{Tables~\ref{tab:a9-asym} and~\ref{tab:a10-nothink}} examine asymmetric training and non-thinking evaluation on Qwen3-8B; these are Qwen-specific thinking-mode ablations.

\subsection{Main Results}
\label{sec:exp:main}

\paragraph{Results across model scales.}
\paperref{Table~\ref{tab:main}} shows that on Qwen3-8B, \method-FKL with the strong annotator raises five-benchmark AVG from GRPO's $78.75\%$ to $81.51\%$, an absolute gain of $2.76\%$ with a paired 95\% CI of $[0.41,4.91]\%$. GPQA-Diamond accuracy is $58.33$, preserving the base model's $58.27$ within evaluation resolution, compared with GRPO's $53.85$. On Qwen3-1.7B, \method-RKL reaches $60.16\%$ versus GRPO's $57.73\%$, an absolute gain of $2.43\%$, while \method-FKL reaches $58.46\%$. FKL on key spans leads at 8B, RKL on error spans at 1.7B. Cell-wise and paired intervals are in \paperref{Apps.~\ref{app:base-eval} and~\ref{app:paired-uncertainty}}.

\begin{table}[!t]
\centering\small
\setlength{\tabcolsep}{4pt}
\caption{\textbf{Main results with thinking-enabled evaluation.} \method, SDPO, and SRPO use symmetric training. AVG is the unweighted five-benchmark mean. Pink AVG cells flag runs with later validation degradation in \paperref{Fig.~\ref{fig:val_reward}}; all scores use validation-selected checkpoints. \textbf{Bold} / \underline{underline}: best / second-best trained methods per model; base rows are references. \paperref{Appendix~\ref{app:base-eval}} reports cell-wise bootstrap 95\% CIs.}
\label{tab:main}
\begin{tabular}{lcccccc}
\toprule
Method & MATH-500 & AIME 24 & AIME 25 & AMC 23 & GPQA-D & AVG \\
\midrule
\multicolumn{7}{l}{\emph{Qwen3-8B}} \\
\quad \textit{Base (ref.)}                       & \textit{96.80} & \textit{76.25} & \textit{67.50} & \textit{95.94} & \textit{58.27} & \textit{78.95} \\
\quad + GRPO                                     & 97.30 & 77.08 & 68.96 & 96.56 & 53.85 & 78.75 \\
\quad + SDPO                                     & 95.13 & 52.50 & 41.25 & 91.41 & 39.90 & \cellcolor{degradationAvg}64.04 \\
\quad + SRPO                                     & 96.48 & 69.17 & 54.17 & 93.12 & 37.31 & \cellcolor{degradationAvg}70.05 \\
\quad + RLSD                                     & 96.68 & 75.42 & 68.54 & 97.03 & \underline{57.26} & 78.99 \\
\rowcolor{traceRow}
\quad + \method (RKL on $\mathcal{E}_y$)         & \underline{97.48} & \underline{78.54} & \underline{70.42} & \textbf{97.81} & 56.44 & \underline{80.14} \\
\rowcolor{traceRow}
\quad + \method (FKL on $\mathcal{K}_y$)         & \textbf{97.83} & \textbf{80.21} & \textbf{73.54} & \underline{97.66} & \textbf{58.33} & \textbf{81.51} \\
\midrule
\multicolumn{7}{l}{\emph{Qwen3-1.7B}} \\
\quad \textit{Base (ref.)}                       & \textit{91.35} & \textit{43.75} & \textit{35.00} & \textit{86.56} & \textit{37.18} & \textit{58.77} \\
\quad + GRPO                                     & 90.33 & 44.58 & 35.00 & 81.56 & 37.18 & 57.73 \\
\quad + SDPO                                     & 81.73 & 16.25 & 17.50 & 51.25 & 27.41 & \cellcolor{degradationAvg}38.83 \\
\quad + SRPO                                     & 86.02 & 24.17 & 23.75 & 62.81 & 35.82 & \cellcolor{degradationAvg}46.51 \\
\quad + RLSD                                     & \underline{91.12} & \textbf{45.42} & 32.92 & \underline{84.38} & \textbf{38.89} & \underline{58.55} \\
\rowcolor{traceRow}
\quad + \method (RKL on $\mathcal{E}_y$)         & \textbf{92.05} & \underline{44.92} & \textbf{38.33} & \textbf{86.82} & \underline{38.67} & \textbf{60.16} \\
\rowcolor{traceRow}
\quad + \method (FKL on $\mathcal{K}_y$)         & 91.00 & 44.83 & \underline{36.30} & 83.38 & 36.77 & 58.46 \\
\bottomrule
\end{tabular}
\end{table}

\paragraph{Training dynamics.}
\paperref{Figure~\ref{fig:val_reward}} tracks validation accuracy and training reward. Gains persist after distillation ends at step~40 and training continues with pure GRPO; SDPO and SRPO degrade after early peaks in these reproduced runs.

\paragraph{Cross-family results.}
\paperref{Table~\ref{tab:cross-family}} reports $30.13\%$ AVG for \method-RKL with online self-annotation on Llama-3.2-3B-Instruct, improving over GRPO and SDPO by 1.48 and 0.50 percentage points, respectively. On SmolLM3-3B, \method-RKL reaches $71.84\%$ at step~40, exceeding GRPO's selected step-325 checkpoint ($69.45\%$) by 2.39 percentage points. These results extend the benefits of selective self-distillation beyond Qwen and demonstrate substantially improved training-step efficiency on SmolLM3, where \method attains higher accuracy with far fewer training steps.

\begin{figure}[H]
  \centering
  \includegraphics[width=\linewidth]{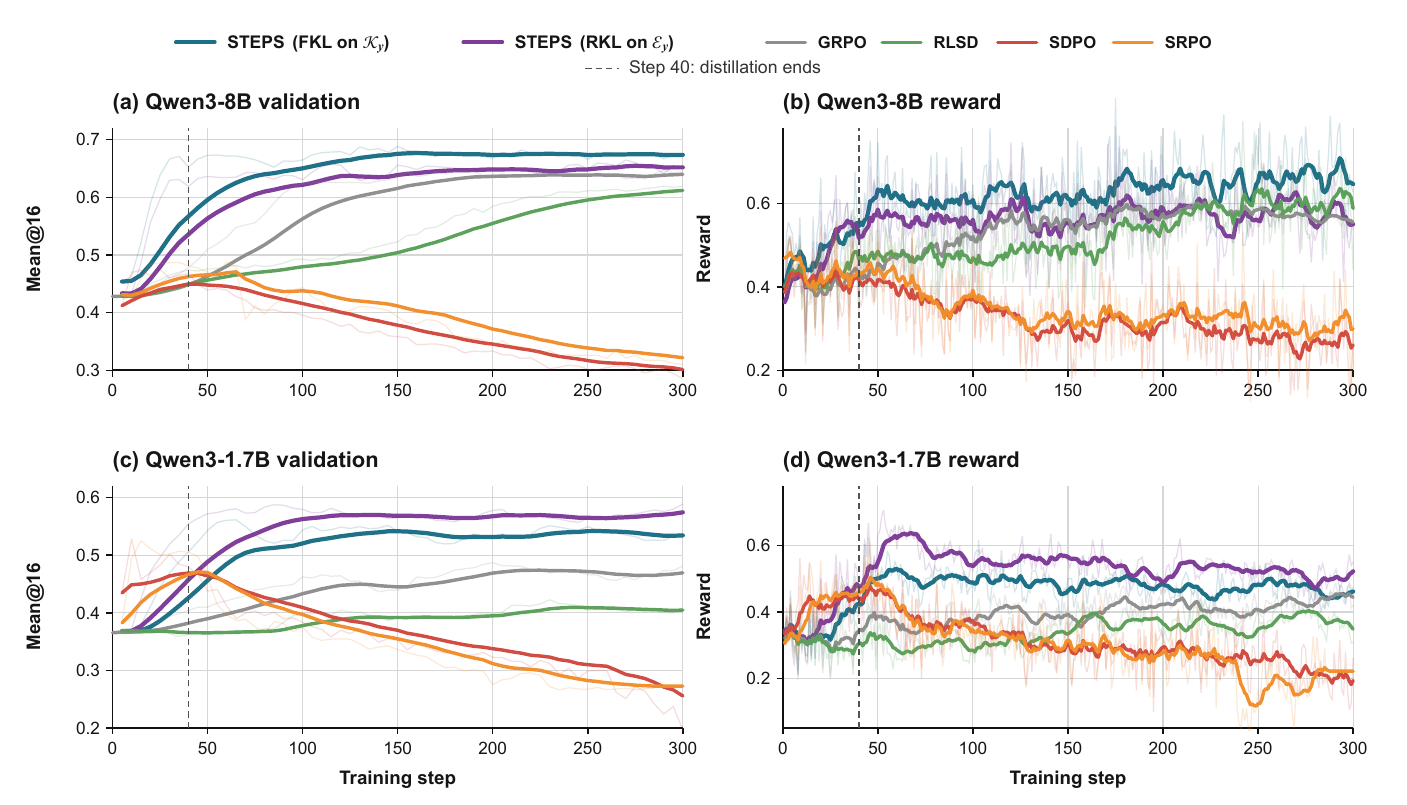}
  \caption{\textbf{Training dynamics.} OpenThoughts validation mean@16 (left) and reward (right): Qwen3-8B (top), Qwen3-1.7B (bottom). EMA $\alpha{=}0.85$. Markers at step~40 indicate the handoff to pure GRPO.}
  \label{fig:val_reward}
\end{figure}

\begin{table}[!ht]
\centering\small
\setlength{\tabcolsep}{5pt}
\caption{\textbf{Cross-family results.} Ranking notation follows \paperref{Tab.~\ref{tab:main}}; SmolLM3 checkpoint steps are shown.}
\label{tab:cross-family}
\begin{tabular}{lccccc}
\toprule
Model / Method & MATH-500 & AIME 24 & AMC 23 & GPQA-D & AVG \\
\midrule
\multicolumn{6}{l}{\emph{Llama-3.2-3B-Instruct}} \\
\quad \textit{Base (ref.)} & \textit{43.7} & \textit{6.2} & \textit{21.3} & \textit{16.8} & \textit{22.00} \\
\quad + GRPO & 52.1 & \textbf{6.7} & 29.5 & 26.3 & 28.65 \\
\quad + SDPO & \textbf{54.4} & \underline{6.3} & \underline{30.4} & \underline{27.4} & \underline{29.63} \\
\rowcolor{traceRow}
\quad + \method-RKL (online-self) & \underline{53.5} & 6.1 & \textbf{30.8} & \textbf{30.1} & \textbf{30.13} \\
\midrule
\multicolumn{6}{l}{\emph{SmolLM3-3B}} \\
\quad \textit{Base (ref.)} & \textit{90.16} & \textit{52.50} & \textit{78.69} & \textit{44.85} & \textit{66.55} \\
\quad + GRPO (step 325) & \underline{94.35} & \underline{54.31} & \underline{79.42} & \underline{49.73} & \underline{69.45} \\
\quad + SDPO & 83.14 & 35.29 & 48.37 & 35.13 & \cellcolor{degradationAvg}50.48 \\
\rowcolor{traceRow}
\quad + \method-RKL (step 40) & \textbf{96.23} & \textbf{56.12} & \textbf{82.31} & \textbf{52.69} & \textbf{71.84} \\
\bottomrule
\end{tabular}
\end{table}

\newpage
\subsection{Ablations}
\label{sec:exp:ablations}

\paragraph{Span Selection.}
\paperref{Figure~\ref{fig:design-ablations}(a)} shows that with RKL and decay fixed on Qwen3-1.7B, semantic error spans outperform random and inverted non-critical masks at matched 25\% coverage, supporting localization beyond sparsity alone. \paperref{Table~\ref{tab:coverage-duration}} in \paperref{Appendix~\ref{app:coverage-duration}} also reports lower held-out AVG for the all-token control than GRPO: $56.68\%$ versus $57.73\%$.

\begin{figure}[!t]
  \centering
  \includegraphics[width=\linewidth]{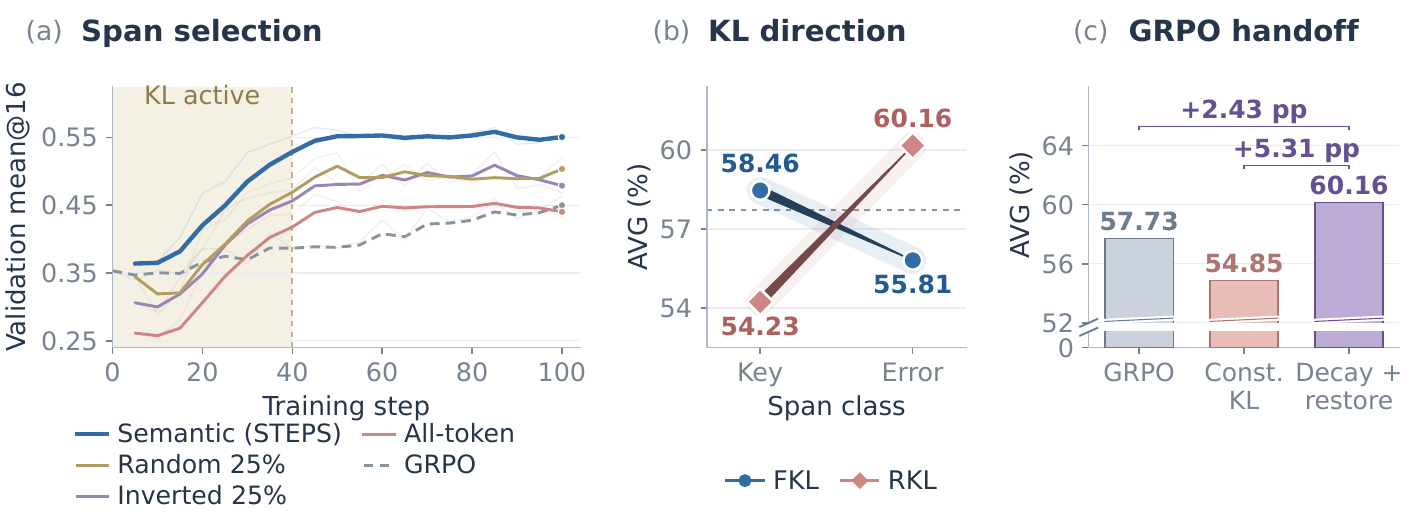}
  \caption{\textbf{Design ablations on Qwen3-1.7B.} (a) Span selection: validation mean@16 (EMA $\alpha\!=\!0.55$), shared decay, and 25\% selective coverage; shading marks the KL-active phase. (b) Span class $\times$ KL direction; the dashed line denotes GRPO. (c) GRPO, constant KL ($\lambda=0.5$), and KL decay with GRPO restoration. Panels (b,c) report single-seed five-benchmark AVG. \paperref{Appendix~\ref{app:coverage-duration}} details matched settings.}
  \label{fig:design-ablations}
\end{figure}

\paragraph{KL Direction.}
Unlike the two configurations in \paperref{Tab.~\ref{tab:main}}, \paperref{Figure~\ref{fig:design-ablations}(b)} separates span class and KL direction under the matched annotator, coverage, and schedule detailed in \paperref{Appendix~\ref{app:coverage-duration}}. FKL performs better on key spans and RKL on error spans, supporting class-specific objectives rather than a uniform KL direction.

\paragraph{GRPO Handoff.}
\paperref{Figure~\ref{fig:design-ablations}(c)} shows that at fixed coverage $\alpha{=}0.25$, holding $\lambda{=}0.5$ lowers Qwen3-1.7B AVG from $60.16$ to $54.85$, supporting the complete handoff design: $\lambda$ controls both distillation and the restoration of GRPO on selected spans. The full coverage and duration sweep is in \paperref{App.~\ref{app:coverage-duration}}.

\paragraph{Token-Level Updates.}
Localizing the KL loss does not automatically localize its effects, since model parameters are shared across token positions. \paperref{Figure~\ref{fig:token-update-localization}} therefore examines realized token-level log-probability changes relative to GRPO. On selected key tokens, \method retains positive shifts, with a slightly larger mean than all-token guidance (0.088 versus 0.082 nats). On unselected tokens, \method remains close to GRPO, whereas all-token guidance produces a wider spread of positive and negative shifts. This small-scale diagnostic complements the selection ablation: the observed updates preserve additional reinforcement at selected positions while limiting departures from GRPO elsewhere.

\clearpage

\begin{figure}[!t]
  \centering
  \includegraphics[width=0.72\linewidth]{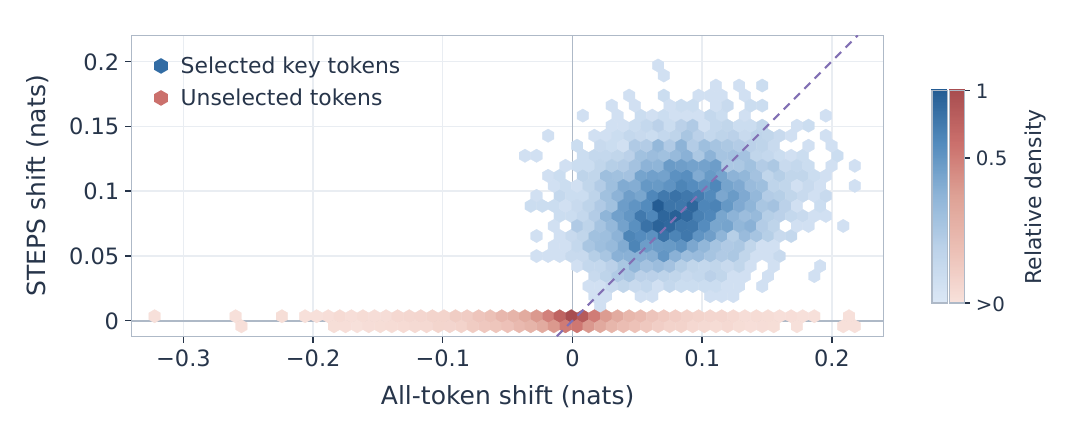}
  \captionsetup{font=small,skip=4pt}
  \caption{\textbf{Token-level update localization.} Post-update log-probability shifts relative to GRPO: all-token guidance (horizontal) versus \method (vertical). Blue and coral hexagons denote selected key and unselected tokens; density is normalized within each group. The dashed line denotes equal shifts. \paperref{Appendix~\ref{app:case-study}} gives diagnostic details.}
  \label{fig:token-update-localization}
\end{figure}

\subsection{Annotation and Cost}
\label{sec:exp:annotation-cost}

\noindent
\begin{minipage}[t]{0.56\linewidth}
\vspace{0pt}
\raggedright
\paperref{Table~\ref{tab:a8-annotator}} shows that the strong annotator achieves the highest AVG for Qwen3-8B FKL-on-key. Online self-annotation uses the training policy without an external annotator, reaching $80.65\%$, an absolute gain of $1.90\%$ over GRPO. Annotation details and training dynamics are in \paperref{App.~\ref{app:annot-dynamics}}.

For this online-self configuration, annotation and teacher forward occur only in the first 40 steps. Component timings in \paperref{Appendix~\ref{app:online-timing}} yield an estimated training-time overhead of $2.7\%$ amortized over 300 steps for online self-annotation.
\end{minipage}\hfill
\begin{minipage}[t]{0.42\linewidth}
\vspace{0pt}
\centering\small
\captionof{table}{\textbf{Annotator comparison} (Qwen3-8B, FKL on key spans). Five-benchmark AVG and absolute gain over GRPO (\%). Paired confidence intervals are reported in \paperref{Appendix~\ref{app:paired-uncertainty}}.}
\label{tab:a8-annotator}
\setlength{\tabcolsep}{4pt}
\begin{tabular}{lrr}
\toprule
Annotator & AVG & $\Delta$ (\%) \\
\midrule
None (GRPO) & 78.75 & --- \\
Strong API & \textbf{81.51} & $\mathbf{+2.76}$ \\
\rowcolor{onlineSelfRow}
Online-self & 80.65 & $+1.90$ \\
Mid (32B) & 79.84 & $+1.09$ \\
Static-self (8B) & 79.67 & $+0.92$ \\
\bottomrule
\end{tabular}
\end{minipage}

\endgroup

\section{Discussion and Conclusion}
\label{sec:discussion}

\method addresses the training collapse observed in SDPO and SRPO under our mathematical reasoning setup. By concentrating span-specific KL guidance on critical reasoning spans during an initial correction phase and then transitioning to GRPO, \method maintains stable training while improving reasoning performance. Experiments across four models from three families show gains in mean accuracy, including with online self-annotation. On SmolLM3, \method achieves 71.84 AVG at step 40, exceeding GRPO's 69.45 at step 325 and demonstrating effective learning with substantially fewer training steps.

These results support treating privileged self-distillation as a targeted, temporary correction mechanism for reasoning RL. Where guidance is applied, which KL direction it uses, and when it is withdrawn are important design choices. \method makes these choices explicit, connecting early corrective supervision with subsequent reward-driven learning. This perspective motivates adaptive span selection and guidance schedules, as well as extensions to broader reasoning tasks.


\clearpage
\bibliographystyle{iclr2027_conference}
\bibliography{references}

\clearpage
\appendix
\section{Notation and Preliminaries}
\label{app:notation}

\paragraph{Notation.}
Let $\pi_\theta$ be a language model parameterized by $\theta$. Given a problem $x$, $y = (y_1, \ldots, y_T)$ is generated autoregressively; $R(x, y) \in \{0, 1\}$ is a binary verifier; $c$ denotes \emph{privileged information} (a verified reasoning trace, environment feedback, or coarse type label). The student is $\pi_S(\cdot \mid x) := \pi_\theta(\cdot \mid x)$ and the privileged-context teacher is $\pi_T(\cdot \mid x, c) := \pi_{\theta_T}(\cdot \mid x, c)$ (same parametric family, with parameters periodically synchronized from the student and a different prompt). Our goal is to train $\pi_\theta$ to maximize verifier reward without incurring the epistemic suppression and OOD degradation observed under all-token OPSD~\citep{kim2026why,yang2026rlsd}.

\paragraph{GRPO baseline and all-token OPSD.}
Group Relative Policy Optimization~\citep{shao2024deepseekmath} samples $G$ rollouts $\{y^{(i)}\}_{i=1}^G \sim \pi_S(\cdot \mid x)$ and assigns each rollout the sequence-level advantage $A^{(i)} = (R^{(i)} - \mu_G) / \sigma_G$; all tokens in $y^{(i)}$ share that advantage. When the group has uniform reward, the zero-variance convention sets $A^{(i)}\equiv0$ on every token; Remark~\ref{rem:deadzone} describes the local auxiliary signal in this dead zone. All-token OPSD~\citep{huebotter2026sdpo,zhao2026opsd} densifies this signal by adding token-level KL guidance $D(\pi_T(\cdot|x,c,\hat y_{<t}) \,\|\, \pi_S(\cdot|x,\hat y_{<t}))$ whose support spans every response position regardless of semantic role, with gradients flowing only through $\pi_S$.

\paragraph{Forward and reverse KL.}
For a fixed rollout prefix, let $q_t := \pi_T(\cdot \mid x,c,\hat y_{<t})$ be the teacher distribution and $p_t := \pi_S(\cdot \mid x,\hat y_{<t})$ be the student distribution. We call $\mathrm{KL}(q_t\|p_t)$ \emph{forward KL} (FKL) and $\mathrm{KL}(p_t\|q_t)$ \emph{reverse KL} (RKL), always relative to teacher-to-student distillation. FKL and RKL are not interchangeable: their per-token gradients depend asymmetrically on $p_t$ and $q_t$, as formalized in Lemma~\ref{lem:softmax} in \paperref{Appendix~\ref{app:proof-softmax}}, so the divergence direction is itself a per-token-class design decision rather than a global hyperparameter.

\paragraph{Critical-span structure.}
For each rollout $y$ we receive a span partition $\{1, \ldots, |y|\} = \mathcal{E}_y \sqcup \mathcal{K}_y \sqcup \mathcal{N}_y$ from the privileged annotator described in Section~\ref{ssec:span-annot}, where $\mathcal{E}_y$ (\emph{error spans}, empty on $R{=}1$ rollouts) and $\mathcal{K}_y$ (\emph{key spans}, empty on $R{=}0$ rollouts) jointly form $\mathcal{S}_y=\mathcal{E}_y\cup\mathcal{K}_y$ and cover at most $\alpha=0.25$ of response tokens; $\mathcal{N}_y$ contains all remaining (non-span) tokens.

\section{Theoretical Statements and Proofs}
\label{app:proofs}

We use the notation of \S\ref{ssec:thy-setup} throughout. All distributions
are on a finite vocabulary $\mathcal{V}$ unless stated otherwise. KL derivatives refer to the idealized unclipped objectives with stop-gradient teacher targets and masks.

\begin{remark}[Span-restricted decomposition]
\label{rem:decomp}
With stop-gradient masks, the partition $\mathcal S_y\sqcup\mathcal N_y=\{1,\ldots,|\hat y|\}$ is fixed during differentiation. Splitting the sequence-normalized token sum in Eq.~\eqref{eq:loss} therefore gives $\nabla_\theta\mathcal L=\nabla_\theta\mathcal L_{\mathrm{span}}+\nabla_\theta\mathcal L_{\mathrm{nonspan}}$. The supports are disjoint in token positions, not in model parameters.
\end{remark}

\subsection{KL Gradient Asymmetry}
\label{app:proof-softmax}

At a fixed rollout position, let \(p_v=\pi_S(v)=\pi_\theta(v)>0\) and \(q_v=\pi_T(v)>0\) denote the student and teacher probabilities on the same causal prefix. Define \(D_F=D_{\mathrm{KL}}(q\|p)=\mathrm{KL}_F(\pi_T,\pi_\theta)\) and \(D_R=D_{\mathrm{KL}}(p\|q)=\mathrm{KL}_R(\pi_T,\pi_\theta)\), with \(q\) fixed during the update. The notation \(r(v)\) used in the derivation below is equivalent to \(r_v\).

\begin{lemma}[Softmax gradient identities]
\label{lem:softmax} Let \(\ell_v\) be the student logit, \(r_v=\log(p_v/q_v)\), and \(\bar r=\sum_v p_vr_v\). Then

\begin{equation*}
\frac{\partial D_F}{\partial\ell_v}=p_v-q_v,
\qquad
\frac{\partial D_R}{\partial\ell_v}=p_v(r_v-\bar r).
\end{equation*}

\end{lemma}

\begin{proof}[Proof of Lemma~\ref{lem:softmax}]
For $\pi_\theta(v) \propto \exp(\ell_v)$ with $Z := \sum_u \exp(\ell_u)$,
\begin{equation}
\frac{\partial \pi_\theta(v')}{\partial \ell_v} \;=\; \pi_\theta(v')\big(\mathbb{1}[v=v'] - \pi_\theta(v)\big),
\qquad
\frac{\partial \log \pi_\theta(v')}{\partial \ell_v} \;=\; \mathbb{1}[v=v'] - \pi_\theta(v).
\label{eq:softmax-derivs}
\end{equation}

\paragraph{Reverse KL.}
$\mathrm{KL}_R(\pi_T, \pi_\theta) = \mathrm{KL}(\pi_\theta \| \pi_T) = \sum_{v'} \pi_\theta(v')(\log \pi_\theta(v') - \log \pi_T(v'))$.
Differentiating:
\begin{align*}
\frac{\partial \mathrm{KL}_R}{\partial \ell_v}
&\;=\; \sum_{v'} \frac{\partial \pi_\theta(v')}{\partial \ell_v}(\log \pi_\theta(v') - \log \pi_T(v'))
+ \sum_{v'} \pi_\theta(v') \frac{\partial \log \pi_\theta(v')}{\partial \ell_v}\\
&\stackrel{\eqref{eq:softmax-derivs}}{=}
\sum_{v'} \pi_\theta(v')\big(\mathbb{1}[v=v'] - \pi_\theta(v)\big) r(v')
+ \sum_{v'} \pi_\theta(v')\big(\mathbb{1}[v=v'] - \pi_\theta(v)\big)\\
&\;=\; \pi_\theta(v) r(v) - \pi_\theta(v) \bar r + \pi_\theta(v) - \pi_\theta(v)\\
&\;=\; \pi_\theta(v) \cdot \big(r(v) - \bar r\big).
\end{align*}
This proves Eq.~\eqref{eq:kl-grads}.

\paragraph{Forward KL.}
$\mathrm{KL}_F(\pi_T, \pi_\theta) = \mathrm{KL}(\pi_T \| \pi_\theta) = -\sum_{v'} \pi_T(v') \log \pi_\theta(v') + \mathrm{const}$.
Differentiating:
\begin{equation*}
\frac{\partial \mathrm{KL}_F}{\partial \ell_v} \;=\; -\sum_{v'} \pi_T(v')\big(\mathbb{1}[v=v'] - \pi_\theta(v)\big)
\;=\; -\pi_T(v) + \pi_\theta(v) \cdot 1
\;=\; \pi_\theta(v) - \pi_T(v).
\end{equation*}
This proves Eq.~\eqref{eq:kl-grads}.
\end{proof}

\begin{corollary}[Pointwise KL asymmetry]
\label{cor:fkl-underalloc} Equation~\eqref{eq:kl-grads} distinguishes two local regimes:
\begin{itemize}[leftmargin=1.1em,itemsep=1pt,topsep=2pt]
\item \textbf{Under-allocated teacher-supported tokens.} For fixed \(q_v>0\) and bounded \(|\bar r|\), as \(p_v\to0\), FKL produces upward logit pressure of magnitude \(q_v-p_v\to q_v\), whereas the RKL magnitude \(p_v|r_v-\bar r|=O(p_v\log(q_v/p_v))\) vanishes.
\item \textbf{Student-overweighted tokens.} When \(p_v>q_v\), FKL produces downward logit pressure of magnitude \(p_v-q_v\). RKL produces downward pressure precisely when \(r_v>\bar r\), with magnitude \(p_v(r_v-\bar r)\). The condition \(p_v>q_v\) alone implies neither this direction nor a larger RKL magnitude.
\end{itemize}
\end{corollary}

\begin{proof}[Proof of Corollary~\ref{cor:fkl-underalloc}]
For fixed $q_v>0$ and bounded $|\bar r|$, $r_v=\log(p_v/q_v)\to-\infty$ and $p_v|r_v-\bar r|=O(p_v\log(q_v/p_v))\to0$, whereas $p_v-q_v\to-q_v$. The signs and magnitudes in the second case follow directly from the two derivatives in Eq.~\eqref{eq:kl-grads}.
\end{proof}

For a direct Euclidean logit step, $\ell_v^{(+)}=\ell_v-\eta\,\partial D/\partial\ell_v$ with $\eta>0$. Thus RKL lowers $\ell_v$ when $r_v>\bar r$, and FKL lowers every logit in a token set whose entries satisfy $p_v>q_v$. These are pointwise logit statements: softmax normalization couples probability changes, and neither monotone token-set mass nor the direction of a shared-parameter optimizer update follows.

\begin{remark}[Auxiliary signal in GRPO dead zones]
\label{rem:deadzone}
If every rollout in a group is correct, the zero-variance convention sets all GRPO advantages to zero. During $\lambda_k>0$, the default FKL branch still has local logit gradient $\lambda_k|\hat y|^{-1}(p_v-q_v)$ at each selected key position by Eq.~\eqref{eq:kl-grads}, which is nonzero as a vector when $p\ne q$. This unclipped, position-wise statement does not ensure a nonzero shared-parameter gradient; the auxiliary signal vanishes when $\lambda_k=0$.
\end{remark}

\subsection{Privileged-Gradient Exposure}
\label{app:proof-exposure}

Following the decomposition of \citet{yang2026rlsd}, fix a step $k$ and condition on $\mathcal{F}_k$: the training history, problem $x_k$, rollout $\hat y_k$, and selected mask $m_k$. Thus the current student and teacher parameters, score vectors, response length, and mask are fixed. The context $c$ follows its conditional distribution given $\mathcal{F}_k$; it need not be independent of the mask. Define
$\bar\pi_{T,k,t}(v):=\mathbb{E}_c[\pi_T(v\mid x_k,c,\hat y_{k,<t})\mid\mathcal{F}_k]$.
For the default FKL-on-key action $(\mu_E,\mu_K)=(0,1)$, the centered privileged-gradient deviation is
\begin{equation}
\delta_{k,t}:=-\sum_v\big(\pi_T(v\mid x_k,c,\hat y_{k,<t})-\bar\pi_{T,k,t}(v)\big)\nabla_\theta\log\pi_S(v\mid x_k,\hat y_{k,<t}),
\label{eq:delta-def}
\end{equation}
with $\mathbb{E}_c[\delta_{k,t}\mid\mathcal{F}_k]=0$ and matching conditional variance
\begin{equation}
V_{k,t}:=\sum_v\mathrm{Var}_c[\pi_T(v\mid x_k,c,\hat y_{k,<t})\mid\mathcal{F}_k]\ge0.
\label{eq:Vt}
\end{equation}
Here $V_{k,t}=0$ if and only if the teacher distribution is almost surely constant under this conditional context distribution; this does not imply that it ignores privileged information across masks or rollouts.

Assume that, uniformly over steps and positions, every zero-sum vector \(a\) satisfies the score-operator bound

\begin{equation}
\left\|\sum_v a_v\nabla_\theta\log\pi_S(v)\right\|_2^2
\leq C_s^2\sum_v a_v^2,
\qquad \sum_v a_v=0.
\label{eq:score-op-bound}
\end{equation}

This is an analytical assumption, distinct from optimizer gradient clipping.

\begin{proposition}[Cumulative exposure bound]
\label{prop:exposure} Let \(m_{k,t}=\mathbb{1}[t\in\mathcal K_{y_k}]\) select the active FKL key-span positions in rollout \(\hat y_k\). Under Equation~\eqref{eq:score-op-bound}, the cumulative weighted second moment satisfies

\begin{equation}
\begin{split}
\mathcal E_K:={}&\sum_{k=1}^{K}\lambda_k^2\,
\mathbb E\!\left[
\frac{1}{|\hat y_k|}\sum_t m_{k,t}\|\delta_{k,t}\|_2^2
\right]\\
\leq{}& C_s^2\sum_{k=1}^{K}\lambda_k^2\,
\mathbb E\!\left[
\frac{1}{|\hat y_k|}\sum_t m_{k,t}V_{k,t}
\right].
\end{split}
\label{eq:exposure-refined}
\end{equation}

If coverage is at most \(\alpha\) and the expected variance averaged over selected positions is bounded by \(\bar V\) at every step, with empty masks contributing zero, then
\[
\mathcal E_K\leq C_s^2\alpha\bar V\sum_{k=1}^{K}\lambda_k^2,
\]
which is the bound reported in Eq.~\eqref{eq:exposure-bound}.
\end{proposition}

\begin{proof}
Let $a_v=\pi_T(v\mid x_k,c,\hat y_{k,<t})-\bar\pi_{T,k,t}(v)$, so $\sum_v a_v=0$ and $\delta_{k,t}=-\sum_v a_v\nabla_\theta\log\pi_S(v)$. Applying the score-operator bound and averaging over context conditional on $\mathcal F_k$ gives
\[
\mathbb E_c[\|\delta_{k,t}\|_2^2\mid\mathcal F_k]
\le C_s^2\sum_v\mathbb E_c[a_v^2\mid\mathcal F_k]
=C_s^2V_{k,t}.
\]
The operator bound controls the full vocabulary sum, including its cross terms. Since the mask and response length are $\mathcal F_k$-measurable, summing over selected positions, taking the outer expectation, and then summing with the deterministic weights $\lambda_k^2$ yields Eq.~\eqref{eq:exposure-refined} by the tower property.

For the coverage bound, let $\mathcal M_k=\{t:m_{k,t}=1\}$ and define $\bar V_{k,\mathrm{span}}=|\mathcal M_k|^{-1}\sum_{t\in\mathcal M_k}V_{k,t}$, with value zero for empty masks. Nonnegativity, $|\mathcal M_k|/|\hat y_k|\le\alpha$, and the uniform assumption $\mathbb E[\bar V_{k,\mathrm{span}}]\le\bar V$ imply
\[
\begin{aligned}
\mathbb E\!\left[\frac{1}{|\hat y_k|}\sum_t m_{k,t}V_{k,t}\right]
&=\mathbb E\!\left[\frac{|\mathcal M_k|}{|\hat y_k|}\bar V_{k,\mathrm{span}}\right]\\
&\le\alpha\,\mathbb E[\bar V_{k,\mathrm{span}}]\le\alpha\bar V.
\end{aligned}
\]
Substitution into Eq.~\eqref{eq:exposure-refined} gives Eq.~\eqref{eq:exposure-bound}. The joint expectation is not factored: no independence of mask, coverage, and context is assumed, and the variance bound concerns selected positions rather than the rollout average.

Finally, the finite schedule in Eq.~\eqref{eq:lambda-schedule} satisfies $0\le\lambda_k\le w_0$ and $\lambda_k=0$ beyond $t_{\mathrm{start}}+T_{\mathrm{decay}}$. Hence $\sum_{k=1}^{K}\lambda_k^2\le(t_{\mathrm{start}}+T_{\mathrm{decay}})w_0^2$ for every $K$, making the exposure upper bound independent of the total training horizon for a fixed guidance schedule.
\end{proof}

\begin{remark}[Scope of the exposure bound]
We bound the cumulative conditional second moment. The vector mean $\mathbb{E}_c[\delta_{k,t}\mid\mathcal{F}_k]$ vanishes by centering, whereas the first absolute moment $\mathbb{E}_c[\|\delta_{k,t}\|\mid\mathcal{F}_k]$ need not vanish. This variance bound does not control mask-dependent shifts in the conditional teacher mean, all privileged-information effects, or downstream optimization error. If the context is determined by the rollout and mask, the conditional variance may be zero even while privileged information affects the update. Finite exposure is therefore not a no-harm or generalization guarantee.
\end{remark}

\subsection{Key-Span Signal Bound and Extensions}
\label{app:proof-alignment}

Use the conditioning $\mathcal F_k$ and fixed teacher targets defined in \paperref{Appendix~\ref{app:proof-exposure}}. At each selected position, $g_{\mathrm{FKL}}(t)=\nabla_\theta D_F(t)$ is the per-token gradient before sequence normalization. The following proposition states the default FKL-on-key case; the proof retains the error-span RKL and joint-branch extensions.

\begin{proposition}[Key-span signal lower bound]
\label{prop:alignment} At step \(k\), let \(\tilde g_k\) be a common \(\mathcal F_k\)-measurable oracle loss-gradient direction, and let \(g_{\mathrm{FKL}}(t)\) denote the per-token FKL gradient. Suppose every nonempty selected key set has expected true-critical fraction at least \(q_K\), conditional on \(\mathcal F_k\). For each selected position, assume that its expected alignment with \(\tilde g_k\), conditional on \(\mathcal F_k\) and its critical status, is at least \(\gamma_K>0\) for true-critical positions and at least \(-B_K\), \(B_K\geq0\), for false positives. Define

\[
p_K=\mathbb E\left[\frac{|\mathcal K_{y_k}|}{|\hat y_k|}\right],
\qquad
g_k^{\mathrm{sel}}=
\frac{\lambda_k}{|\hat y_k|}\sum_{t\in\mathcal K_{y_k}}g_{\mathrm{FKL}}(t),
\]

with empty sets contributing zero. Then

\[
\mathbb E\left[\left\langle g_k^{\mathrm{sel}},\tilde g_k\right\rangle\right]
\geq\lambda_kp_K\left[q_K\gamma_K-(1-q_K)B_K\right].
\]

For \(\lambda_kp_K>0\), this lower bound is positive whenever \(q_K>B_K/(\gamma_K+B_K)\).
\end{proposition}

\begin{proof}
We work under binary span weights $w_t\in\{0,1\}$ as in \S\ref{ssec:span-annot}. Under action profile $(\mu_E,\mu_K)\in\{0,1\}^2$, the per-step selected gradient is
\begin{equation*}
g_k^\mathrm{sel}=\frac{\lambda_k}{|\hat y_k|}
\left(\mu_E\sum_{t\in\mathcal{E}_{y_k}}g_\mathrm{RKL}(t)
+\mu_K\sum_{t\in\mathcal{K}_{y_k}}g_\mathrm{FKL}(t)\right),
\end{equation*}
i.e., a sequence-normalized sum matching Eq.~\eqref{eq:loss}. Equivalently, each per-set mean is multiplied by its fraction of the full rollout. Empty sets contribute zero.

Condition on $\mathcal{F}_k$ and use the same $\mathcal{F}_k$-measurable oracle direction $\tilde g_k$ at every selected position. For a class $C\in\{E,K\}$, let $Z_t^C\in\{0,1\}$ indicate true class membership at a selected position and write $g_E=g_\mathrm{RKL}$, $g_K=g_\mathrm{FKL}$. The conditional margin assumptions are
\begin{equation*}
\mathbb{E}[\langle g_C(t),\tilde g_k\rangle\mid\mathcal{F}_k,Z_t^C]
\ge Z_t^C\gamma_C-(1-Z_t^C)B_C.
\end{equation*}
For each nonempty selected set $\mathcal{C}_{y_k}$, require the uniform conditional precision bound
\begin{equation*}
\frac{1}{|\mathcal{C}_{y_k}|}\sum_{t\in\mathcal{C}_{y_k}}
\Pr(Z_t^C=1\mid\mathcal{F}_k)\ge q_C.
\end{equation*}
Taking expectation over $Z_t^C$ and summing the margins against the common direction gives
\begin{equation*}
\mathbb{E}\!\left[\left\langle\frac{1}{|\hat y_k|}\sum_{t\in\mathcal{C}_{y_k}}g_C(t),\tilde g_k\right\rangle\middle|\mathcal{F}_k\right]
\ge\frac{|\mathcal{C}_{y_k}|}{|\hat y_k|}\big(q_C\gamma_C-(1-q_C)B_C\big).
\end{equation*}
Taking the outer expectation and writing $p_C:=\mathbb{E}[|\mathcal{C}_{y_k}|/|\hat y_k|]$ at step $k$ yields
\begin{equation*}
\mathbb{E}[\langle g_k^\mathrm{sel},\tilde g_k\rangle]
\ge\lambda_k\Big[\mu_Ep_E\big(q_E\gamma_E-(1-q_E)B_E\big)
+\mu_Kp_K\big(q_K\gamma_K-(1-q_K)B_K\big)\Big].
\end{equation*}
The default $(\mu_E,\mu_K)=(0,1)$ gives the proposition. The lower bound is positive if every active class with positive coverage has a positive bracket and at least one has $\lambda_kp_C>0$. Independence between token positions, coverage, and context is unnecessary because the precision and margin assumptions hold uniformly under the conditioning.
\end{proof}

\begin{remark}[On the oracle direction]
$\tilde g_k$ is an unobservable verifier-grounded loss-gradient direction. Prop.~\ref{prop:alignment} is a \emph{conditional} lower bound under the stated precision and alignment assumptions. Alignment with different token-specific directions $\tilde g(t)$ would not imply the sequence-level statement; the shared direction is essential. Neither the local softmax identities nor the empirical token-lift diagnostic verifies these assumptions. This is a structural complement to the conditional exposure bound, not an optimality theorem for the trained network or a guarantee for the clipped optimizer step.
\end{remark}

\subsection{Per-token damping under teacher coverage gap}
\label{app:proof-rlsd}

\begin{proposition}[Per-token damping]
\label{prop:rlsd}
Let $w_t = \pi_T(y_t|x,c,y_{<t}) / \pi_S(y_t|x,y_{<t})$ be the unclipped positive-advantage reweighting factor considered here, with $A>0$. If $\pi_T(y_{t^\star}|c) \le \delta$ and $\pi_S(y_{t^\star}) \ge p_0>0$ at some position $t^\star$, then $\hat A_{t^\star} := A \cdot w_{t^\star} \le A\delta/p_0$.
\end{proposition}

The condition is local --- when $\delta<p_0$, one token with low teacher probability suffices for damping --- and can arise in math reasoning, where the privileged context is typically a single canonical solution and the student may sample valid alternatives. Independent of~\citet{yang2026rlsd}'s mutual-information ill-posedness; the two viewpoints are complementary.

\begin{proof}
By definition,
\begin{equation*}
w_{t^\star} \;=\; \frac{\pi_T(y_{t^\star} \mid x, c, y_{<t^\star})}{\pi_S(y_{t^\star} \mid x, y_{<t^\star})}
\;\le\; \frac{\delta}{p_0},
\end{equation*}
where the inequality applies the two probability bounds in the proposition.
Multiplying both sides by $A > 0$:
$\hat A_{t^\star} = A \cdot w_{t^\star} \le A \cdot \delta / p_0$.
\end{proof}

\begin{remark}[On the unclipped form]
The proposition is stated for the \emph{unclipped} reweighting factor
$w_t = \pi_T / \pi_S$. If this factor is clipped to $[1-\epsilon_w,1+\epsilon_w]$, with $0<\epsilon_w<1$, the per-token deviation is bounded.
Under clipping, the weight is bounded below by $1-\epsilon_w$, so the unclipped upper bound $\delta/p_0$ need not hold; the qualitative phenomenon --- that $w_t$ \emph{still} damps
correct non-canonical tokens whenever $\pi_T(y_t) < \pi_S(y_t)$ --- persists
but is bounded by the clip.
\end{remark}

\begin{remark}[Why this is not a global collapse theorem]
Prop.~\ref{prop:rlsd} is a per-token statement: at any single position
satisfying the hypotheses with $\delta<p_0$, the reweighting factor damps the positive-advantage term on that token.
We do \emph{not} prove that this damping aggregates to global collapse of the
training objective. The aggregate training behavior depends on the remaining token updates, clipping, and data; it is an empirical question rather than a consequence of this bound.
\end{remark}

\section{Experimental Details}
\label{app:hparams}
\begingroup
\captionsetup[table]{font=small,skip=5pt}

\setlength{\textfloatsep}{12pt plus 2pt minus 2pt}
\setlength{\floatsep}{8pt plus 2pt minus 2pt}
\setlength{\intextsep}{6pt plus 1pt minus 1pt}

\subsection{Training Hyperparameters}
\label{app:hparams-train}

\paragraph{Common settings.}
The Qwen3-8B experiments share the following settings: Qwen3-8B base, OpenThoughts-114k math subset (up to 30K problem--solution pairs), H100 GPUs, AdamW optimizer with weight decay $0.01$ and gradient-clip norm $1.0$, max prompt length $2048$, vLLM rollout at $T = 0.6$, train batch size $32$ with $G = 8$ rollouts per problem, mini-batch size $32$, learning rate $1{\times}10^{-5}$ with $10$-step linear warm-up, and $300$ total training steps. Max response length is $32768$ for Think-student training and $8192$ for NoThink-student ablations. The cross-family runs reuse the optimizer, rollout, and evaluation settings, with model-specific details in \paperref{App.~\ref{app:cross-family}}. Method-specific deltas are listed in \paperref{Tab.~\ref{tab:hp-baselines}} (baselines) and \paperref{Tab.~\ref{tab:hp-trace}} (\method); best-checkpoint indices used for \paperref{Tab.~\ref{tab:main}} are \method step $250$, GRPO step $250$, SDPO step $35$, SRPO step $65$ (collapse onset).

\begin{table}[H]
\centering\footnotesize
\setlength{\tabcolsep}{4pt}
\caption{Method-specific hyperparameters for the four baselines. ``---'' marks parameters not applicable to the method. All baselines use AdamW with the common settings above and learning rate $1{\times}10^{-5}$.}
\label{tab:hp-baselines}
\begin{tabular}{@{}
>{\raggedright\arraybackslash}p{\dimexpr0.20\linewidth-1.6\tabcolsep\relax}
>{\raggedright\arraybackslash}p{\dimexpr0.22\linewidth-1.6\tabcolsep\relax}
>{\raggedright\arraybackslash}p{\dimexpr0.19\linewidth-1.6\tabcolsep\relax}
>{\raggedright\arraybackslash}p{\dimexpr0.18\linewidth-1.6\tabcolsep\relax}
>{\raggedright\arraybackslash}p{\dimexpr0.21\linewidth-1.6\tabcolsep\relax}
@{}}
\toprule
& GRPO & SDPO~\citep{huebotter2026sdpo} & SRPO~\citep{xx2026srpo} & RLSD~\citep{yang2026rlsd} \\
\midrule
Distillation divergence & --- & JSD ($\alpha{=}0.5$) & FKL & --- (no KL grad) \\
Distillation top-$K$    & --- & 100 & 100 & --- \\
Teacher / EMA           & --- & frozen (EMA $0$) & EMA $0.05$ & sync $N{=}10$ \\
$\lambda$ schedule      & --- & constant & constant & $1.0{\to}0$ over $50$ steps \\
Other                   & DAPO $\varepsilon{=}0.2/0.28$, IS-clip $2$ & correct-sibling rollouts & entropy suppression $\beta{=}1$ & teacher/student-ratio advantage weighting \\
\bottomrule
\end{tabular}
\end{table}

\begin{table}[H]
\centering\small
\caption{\method-specific hyperparameters (default $(\mu_E, \mu_K) = (0, 1)$, FKL on $\mathcal{K}_y$ only).}
\label{tab:hp-trace}
\begin{tabular}{lll}
\toprule
\textbf{Category} & \textbf{Parameter} & \textbf{Value} \\
\midrule
Loss       & Distillation divergence (default)  & FKL on $\mathcal{K}_y$ \\
           & Distillation top-$K$               & 100 \\
           & Per-vocab KL clip ($\tau$)         & 0.05 \\
           & Coverage cap ($\alpha$)            & 0.25 \\
           & DAPO $\varepsilon_\mathrm{low} / \varepsilon_\mathrm{high}$ & $0.2 / 0.28$ \\
\midrule
Schedule   & KL decay start ($t_\mathrm{start}$) & 10 \\
           & KL decay window ($T_\mathrm{decay}$)  & 30 \\
           & Initial KL weight ($w_0$)             & 0.5 \\
           & Teacher sync interval ($N$)           & 10 \\
\midrule
Routing    & Default $(\mu_E, \mu_K)$              & $(0, 1)$ \\
           & Thinking mode                         & symmetric (student / teacher both Think) \\
           & Annotator                             & \texttt{qwen3.5-plus}, thinking disabled \\
\bottomrule
\end{tabular}
\end{table}

\subsection{Algorithm}
\label{app:algorithm}

Algorithm~\ref{alg:trace} gives the per-mini-batch update loop. When $\lambda_k=0$, both annotation and the teacher forward are skipped before either is invoked, and all response tokens receive the GRPO objective. Otherwise, the annotator $\pi_A$ runs once per rollout to produce span masks and a coarse type label, the teacher $\pi_T$ then forwards in thinking mode conditioned on $(x, c^{(i)})$, and the routed mixed-KL loss is added to the GRPO loss with weight $\lambda_k$ during the warm-up / decay window; after step $t_\mathrm{start} + T_\mathrm{decay}$ the method reduces to pure GRPO.

\begin{algorithm}[H]
\small
\caption{\method (one mini-batch update)}
\label{alg:trace}
\begin{algorithmic}[1]
\Require policy $\pi_\theta$, teacher $\pi_T$, annotator $\pi_A$, current step $k$
\State Sample $G$ rollouts $\{\hat y^{(i)}\}_{i=1}^G \sim \pi_S(\cdot|x)$; compute $R^{(i)}$ and GRPO advantages $\hat A^{(i)}$
\State Compute $\lambda_k$ from Eq.~\eqref{eq:lambda-schedule}; set $\rho_k = 1-\lambda_k/w_0$
\If{$\lambda_k = 0$}
  \State Update $\theta$ using GRPO on all response tokens
  \Comment{skip annotation and teacher forward}
\Else
  \For{each rollout $\hat y^{(i)}$}
    \State $\mathcal{E}_{\hat y^{(i)}}$ or $\mathcal{K}_{\hat y^{(i)}}$, type label $\tau^{(i)} \gets \pi_A(x, \hat y^{(i)}, R^{(i)})$
    \State Project character spans to token mask $m^{(i)}$, cap at $\alpha = 0.25$
  \EndFor
  \State Forward $\pi_T$ in Think mode with $(x,c^{(i)})$ on each causal prefix $\hat y^{(i)}_{<t}$, where $c^{(i)}=\tau^{(i)}$
  \State Forward $\pi_S$ in Think mode on the same causal prefixes (default symmetric configuration)
  \State Compute routed KL and GRPO losses per Eq.~\eqref{eq:loss}
  \State Update $\theta$ with $\mathcal{L}_\mathrm{GRPO}^{(\mathcal{N})} + \rho_k\mathcal{L}_\mathrm{GRPO}^{(\mathcal{S})}$
  \Statex \hspace{\algorithmicindent}\hspace{\algorithmicindent} $\qquad {}+\lambda_k[\mu_E\mathcal{L}_\mathrm{RKL}^{(\mathcal{E})}+\mu_K\mathcal{L}_\mathrm{FKL}^{(\mathcal{K})}]$
\EndIf
\If{$k \bmod N = 0$ \textbf{and} $k < t_\mathrm{start}+T_\mathrm{decay}$}
  \State $\theta_T \gets \theta$ \Comment{periodic teacher sync while future KL remains}
\EndIf
\end{algorithmic}
\end{algorithm}

\clearpage
\subsection{Annotator and Teacher Prompts}
\label{app:prompts}

The privileged annotator $\pi_A$ (default: \texttt{qwen3.5-plus} API) receives a (problem, rollout, verifier-outcome) triple and returns a JSON span specification. Annotator thinking mode is disabled; teacher thinking mode is enabled. The four prompt templates used for these roles are reproduced below.

\begin{promptblock}{Annotator $\pi_A$ \,---\, wrong-rollout error-span prompt}{promptInstrTitle}{promptInstrBg}
You are a mathematics error locator. The student's solution is shown as numbered segments \sk{[0]}, \sk{[1]}, \sk{[2]}, $\dots$\, Identify the segments that contain the root-cause reasoning errors.

\textbf{Output} (strict JSON):\\
\sk{\{ error\_spans: [\,\{ segment\_ids: [...], error\_type \}\,] \}}

\sk{error\_type} $\in$ \{\,\sk{missed\_case}, \sk{illegal\_step}, \sk{wrong\_constraint}, \sk{wrong\_equivalence}, \sk{premature\_conclusion}, \sk{format\_error}, \sk{arithmetic\_slip}, \sk{sign\_error}, \sk{off\_by\_one}\,\}.

\textbf{Rules.}
Each span is 1--3 consecutive \sk{segment\_ids}. Identify 0--3 spans, focusing on root-cause errors rather than downstream propagation of an upstream mistake. For omissions, select the segment immediately before the gap. Return an empty list if no segment can be confidently identified as root-causal --- do not guess. Output only the JSON; do not explain, rewrite, or solve the problem.
\end{promptblock}

\begin{promptblock}{Annotator $\pi_A$ \,---\, correct-rollout key-span prompt}{promptKeyTitle}{promptKeyBg}
You are a mathematics reasoning analyst. The student's solution is shown as numbered segments \sk{[0]}, \sk{[1]}, \sk{[2]}, $\dots$\, Identify the segments that carry the decisive reasoning steps --- the critical decisions or insights that make this solution work.

\textbf{Output} (strict JSON):\\
\sk{\{ key\_spans: [\,\{ segment\_ids: [...], step\_type \}\,] \}}

\sk{step\_type} $\in$ \{\,\sk{case\_split}, \sk{boundary\_check}, \sk{invariant}, \sk{substitution}, \sk{constraint\_use}, \sk{symmetry}, \sk{construction\_step}, \sk{final\_verification}, \sk{key\_formula}, \sk{insight}\,\}.

\textbf{Rules.}
Each span is 1--3 consecutive \sk{segment\_ids}. Identify 0--3 spans, skipping routine algebra, restatements, and final-answer formatting. Return an empty list if no segment is confidently critical --- do not pad with low-importance steps. Output only the JSON; do not explain or solve the problem.
\end{promptblock}

\begin{promptblock}{Teacher $\pi_T$ \,---\, wrong-rollout diagnostic prefix}{promptRubricTitle}{promptRubricBg}
\sk{\{ problem \}}

\textit{[Private diagnostic context. Do not mention this context in the response.]}

A sampled solution to this problem contains a local reasoning issue of the following type:\\[2pt]
\hspace*{1.5em}\sk{-- \{ error\_descriptions \}}

When solving, be especially careful around this type of issue. Do not reveal, quote, or locate the diagnostic context. Solve the problem independently and end with a boxed answer.
\end{promptblock}

\begingroup
\mdfsetup{nobreak=true}
\begin{promptblock}{Teacher $\pi_T$ \,---\, correct-rollout diagnostic prefix}{promptInputTitle}{promptInputBg}
\sk{\{ problem \}}

\textit{[Private diagnostic context. Do not mention this context in the response.]}

A successful solution to this problem relies on reasoning steps of the following type:\\[2pt]
\hspace*{1.5em}\sk{-- \{ step\_descriptions \}}

When solving, make the corresponding reasoning explicit enough to support the answer, but do not use checklist language or refer to this context. End with a boxed answer.
\end{promptblock}
\endgroup

The teacher therefore receives only the original problem $x$ plus a coarse type label; no rollout text, span locations, or span content are passed in as additional privileged context. As in standard on-policy distillation it then evaluates token logits causally on the sampled prefix $\hat y_{<t}$.

\subsection{Span-to-Token Alignment}
\label{app:span-token-alignment}
Given character-level spans returned by $\pi_A$ over the response text, we project to a token boolean mask $m \in \{0, 1\}^{|\hat y|}$ as follows: for each token $y_t$ we compute its character interval $[c_t^-, c_t^+)$ by greedy alignment of single-token decodes against the original response (avoiding decode--re-encode round-trip drift); a token is marked if its interval intersects any span. We then enforce $|\{t : m_t = 1\}| \le \alpha |\hat y|$ with $\alpha = 0.25$ by keeping the top-weight tokens.

\subsection{Decoding Configurations}
\label{app:decoding}
We use two decoding protocols. The \textbf{main thinking-mode protocol} uses $T{=}0.6$, $p{=}0.95$, $k{=}20$, $\min p{=}0$, and max\_new\_tokens$=38912$, with \texttt{avg@8} on MATH-500 and GPQA-Diamond and \texttt{avg@16} on AIME and AMC. It applies to the Qwen3 main results, annotator-quality and asymmetric-training ablations, and OOD evaluation. The \textbf{non-thinking protocol} disables thinking, uses max\_new\_tokens$=8192$ and \texttt{avg@8} on all benchmarks, and applies only to the NoThink-evaluation ablation in \paperref{Table~\ref{tab:a10-nothink}}.

\subsection{Training and Evaluation Modes on Qwen Models}
\label{app:thinking-modes}

We distinguish three choices in the Qwen experiments: the student's rollout mode during training, the self-teacher's conditioning mode, and the student's decoding mode during evaluation. \paperref{Table~\ref{tab:thinking-modes}} summarizes the main configuration and the Qwen3-8B thinking-mode ablations.

\begin{table}[!htbp]
\centering\small
\setlength{\tabcolsep}{4pt}
\caption{Qwen3 training and evaluation modes. The NoThink-evaluation row reuses the asymmetric-training checkpoints.}
\label{tab:thinking-modes}
\begin{tabular}{llll}
\toprule
Configuration & Training student & Teacher & Evaluation \\
\midrule
Main \method, SDPO, and SRPO & Think & Think & Think \\
Asymmetric ablation in \paperref{Table~\ref{tab:a9-asym}} & NoThink & Think & Think \\
NoThink-eval ablation in \paperref{Table~\ref{tab:a10-nothink}} & NoThink & Think & NoThink \\
\bottomrule
\end{tabular}
\end{table}

GRPO has no teacher; its row in \paperref{Table~\ref{tab:a9-asym}} reuses the reference from \paperref{Table~\ref{tab:main}}.

\subsection{Cross-Family Configuration Details}
\label{app:cross-family}

Llama-3.2-3B-Instruct and SmolLM3-3B use RKL on error spans, with the optimizer and rollout settings of \paperref{Appendix~\ref{app:hparams-train}} and the main decoding protocol of \paperref{Appendix~\ref{app:decoding}}. Llama uses online self-annotation; SmolLM3 uses the strong annotator, \texttt{qwen3.5-plus}. Checkpoint selection follows the validation mean@16 rule in Section~\ref{sec:exp:setup}.

For both families, AVG covers MATH-500, AIME 2024, AMC 2023, and GPQA-Diamond, omitting AIME 2025 from the Qwen3 five-benchmark mean. The reported SmolLM3 checkpoints are step 40 for \method and step 325 for GRPO.
\endgroup

\section{Additional Results}
\label{app:full-tables}

\setlength{\textfloatsep}{12pt plus 2pt minus 2pt}
\setlength{\floatsep}{10pt plus 2pt minus 2pt}
\setlength{\intextsep}{8pt plus 2pt minus 2pt}

\subsection{Comparison of Post-Training Paradigms}
\label{app:paradigms}

\paperref{Table~\ref{tab:paradigms}} summarizes the source of supervision, teacher provenance, privileged information, and token support for the paradigms discussed in Section~\ref{sec:related}.

\begin{table}[!htbp]
\centering
\footnotesize
\setlength{\tabcolsep}{4pt}
\renewcommand{\arraystretch}{1.06}
\caption{\textbf{Comparison of post-training paradigms.}
\textbf{Bold} marks \method's distinctive choices.}
\label{tab:paradigms}
\begin{fullwidthtablebody}{lllll}{5}
\toprule
Method & Signal & Teacher & Extra info & KL support \\
\midrule
SFT & labels & none & gold response & demo tokens \\
GRPO & verifier $R$ & none & none & no teacher KL \\
OPD/GKD & logits & external / frozen & info-symmetric & all tokens \\
SDPO & logits & self, EMA & feedback + sibling & all tokens \\
SRPO & $R$ + logits & self, EMA & feedback + sibling & branch-weighted \\
RLSD & $R \times \tilde w_t$ & self, sync $N{=}10$ & verified trace $r$ & all tokens \\
\rowcolor{traceRow}
\textbf{\method} & \textbf{$R$ + span KL} & \textbf{self, sync $N{=}10$} & \textbf{coarse type only} & \textbf{critical spans} \\
\bottomrule
\end{fullwidthtablebody}
\end{table}

\FloatBarrier

\subsection{Key-token Lift and Credit Assignment}
\label{app:case-study}

Both the key-token lift analysis in \paperref{Table~\ref{tab:lift-main}} and the token-level update localization analysis in \paperref{Figure~\ref{fig:token-update-localization}} use Qwen3-8B and follow the experimental setup described in Section~\ref{sec:exp:setup} and \paperref{Appendix~\ref{app:hparams-train}}, including the starting checkpoint, number of training updates, and training-batch configuration.

\paragraph{Key-token lift.}
\label{ssec:thy-proxies} To examine whether the local effect in Corollary~\ref{cor:fkl-underalloc} is reflected in actual updates, we measure post-update log-probability lift on a common held-out set of teacher-supported key-span tokens. Let $\mathcal T_{\mathrm{lift}}$ denote this set, selected before the update by requiring $t\in\mathcal K_y$ and $\pi_T^{(k)}(\hat y_t)>\pi_{\theta_k}(\hat y_t)$. We define

\[
\Delta_{\mathrm{lift}}
=\mathbb E_{t\in\mathcal T_{\mathrm{lift}}}\!\bigl[
\log\pi_{\theta_{k+1}}(\hat y_t)-\log\pi_{\theta_k}(\hat y_t)
\bigr].
\]

\paperref{Table~\ref{tab:lift-proxy}} shows the largest lift for \method-FKL, consistent with the local mechanism. This diagnostic measures realized token-probability changes, not the oracle alignment assumed in Proposition~\ref{prop:alignment}.

\begin{table}[H]
\centering\small
\setlength{\tabcolsep}{6pt}
\caption{Post-update key-token log-probability lift on a shared held-out token set.}
\label{tab:lift-proxy}
\begin{tabular}{l >{\columncolor{traceRow}}r >{\columncolor{traceRow}}r rrrr}
\toprule
Method & \method-FKL & \method-RKL & RLSD & GRPO & SRPO & SDPO \\
\midrule
$\Delta_{\mathrm{lift}}$ (nats) & $+0.145$ & $+0.078$ & $+0.068$ & $+0.054$ & $+0.027$ & $+0.019$ \\
vs. GRPO & $+168\%$ & $+44\%$ & $+26\%$ & --- & $-50\%$ & $-65\%$ \\
\bottomrule
\end{tabular}
\end{table}

\paragraph{Token-level update localization.}
\label{app:token-update-localization}
\paperref{Figure~\ref{fig:token-update-localization}} compares paired post-update log-probability shifts relative to GRPO, $d_t^{(m)}=\log p_t^{(m)}-\log p_t^{(\mathrm{GRPO})}$, where $m$ is \method or all-token guidance. The diagnostic comprises 22,715 tokens from eight responses to four problems. The horizontal and vertical axes show all-token and \method shifts, respectively. Blue and coral hexagons represent selected key tokens and unselected tokens, with density normalized separately within each group. Points above the dashed equality line have larger shifts under \method.

Selected-token mean shifts are 0.088 nats for \method and 0.082 nats for all-token guidance; unselected-token mean absolute shifts are 0.00052 and 0.02425 nats, respectively. In this diagnostic sample, \method retains slightly larger mean shifts on selected tokens while limiting departures from GRPO elsewhere. Since model parameters are shared across positions, this localization is an observed property of the realized updates, not an automatic consequence of masking the KL loss.

\paragraph{Token-level credit assignment.}

\paperref{Figure~\ref{fig:case-study}} visualizes token-level credit on a Diophantine case study ($6^a = 1 + 2^b + 3^c$) for one incorrect and one correct rollout. Credit profiles are \emph{measured} per-token update magnitudes for RLSD, SDPO, and \method (one optimization step on the same rollout, same seed); the right panels report concentration inside annotator spans versus outside.

\begin{figure}[!ht]
\centering
\includegraphics[width=\linewidth]{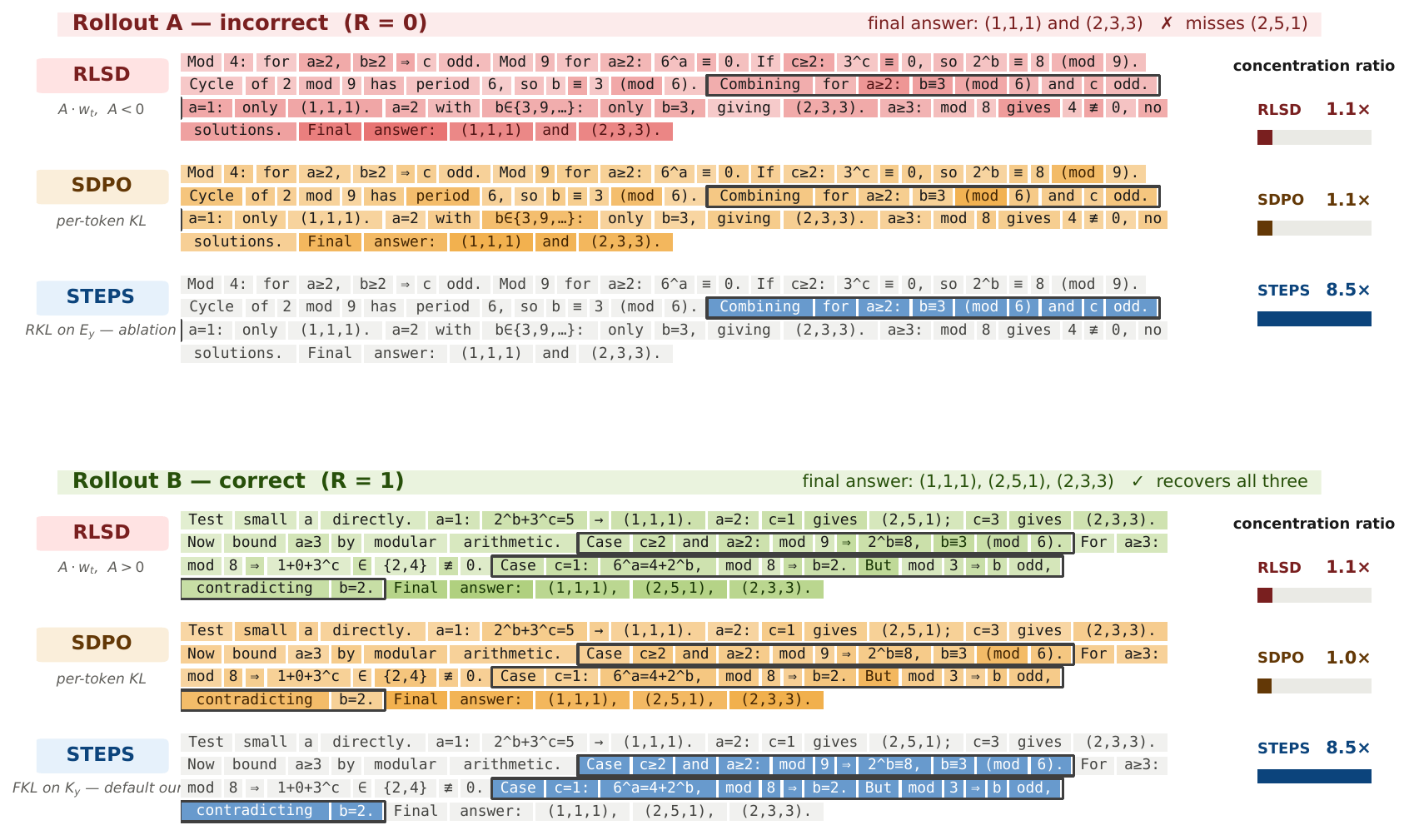}
\captionsetup{font=small,width=\linewidth}
\caption{Token-level credit assignment on the Diophantine case study. Black boxes denote annotator spans overlaid for all methods; only \method uses them as loss masks. Credit is normalized update magnitude, and the concentration ratio is mean credit inside boxes divided by outside. Rollout~A shows the optional RKL-on-$E_y$ ablation; Rollout~B shows the default FKL-on-$K_y$ configuration.}
\label{fig:case-study}
\end{figure}

\subsection{Annotator Quality and Dynamics}
\label{app:annot-dynamics}

\paperref{Table~\ref{tab:annotator-full}} reports the benchmark-level results summarized in \paperref{Table~\ref{tab:a8-annotator}}. Online-self reuses the actor weights through vLLM on the same H100 during the initial 40-step phase. Static-self uses the fixed Qwen3-8B base checkpoint, zero-shot with the same annotation template.

\begin{table}[!htbp]
\centering\small
\caption{\textbf{Full annotator comparison on Qwen3-8B (FKL on key spans).} AVG is the five-benchmark mean; $\Delta$ is the absolute gain over GRPO, expressed in \%. Best per column is \textbf{bold}. Strong API: \texttt{qwen3.5-plus}; mid: Qwen3-32B; static-self: the Qwen3-8B base checkpoint. \paperref{Appendix~\ref{app:paired-uncertainty}} reports the online-self paired 95\% CI, $[-0.25,3.75]\%$, which includes zero.}
\label{tab:annotator-full}
\setlength{\tabcolsep}{4pt}
\begin{fullwidthtablebody}{lccccccc}{8}
\toprule
Annotator & MATH-500 & AIME 24 & AIME 25 & AMC 23 & GPQA-D & AVG & $\Delta$ (\%) \\
\midrule
None (GRPO) & 97.30 & 77.08 & 68.96 & 96.56 & 53.85 & 78.75 & --- \\
\midrule
Strong API & \textbf{97.83} & \textbf{80.21} & \textbf{73.54} & \textbf{97.66} & \textbf{58.33} & \textbf{81.51} & $\mathbf{+2.76}$ \\
\rowcolor{onlineSelfRow}
Online-self & 97.55 & 78.54 & 72.29 & 97.03 & 57.84 & 80.65 & $+1.90$ \\
Mid (32B) & 96.83 & 79.08 & 68.61 & 97.27 & 57.43 & 79.84 & $+1.09$ \\
Static-self (8B) & 97.34 & 75.92 & 72.19 & 96.11 & 56.78 & 79.67 & $+0.92$ \\
\bottomrule
\end{fullwidthtablebody}
\end{table}

\begin{figure}[!ht]
  \centering
  \includegraphics[width=0.72\linewidth]{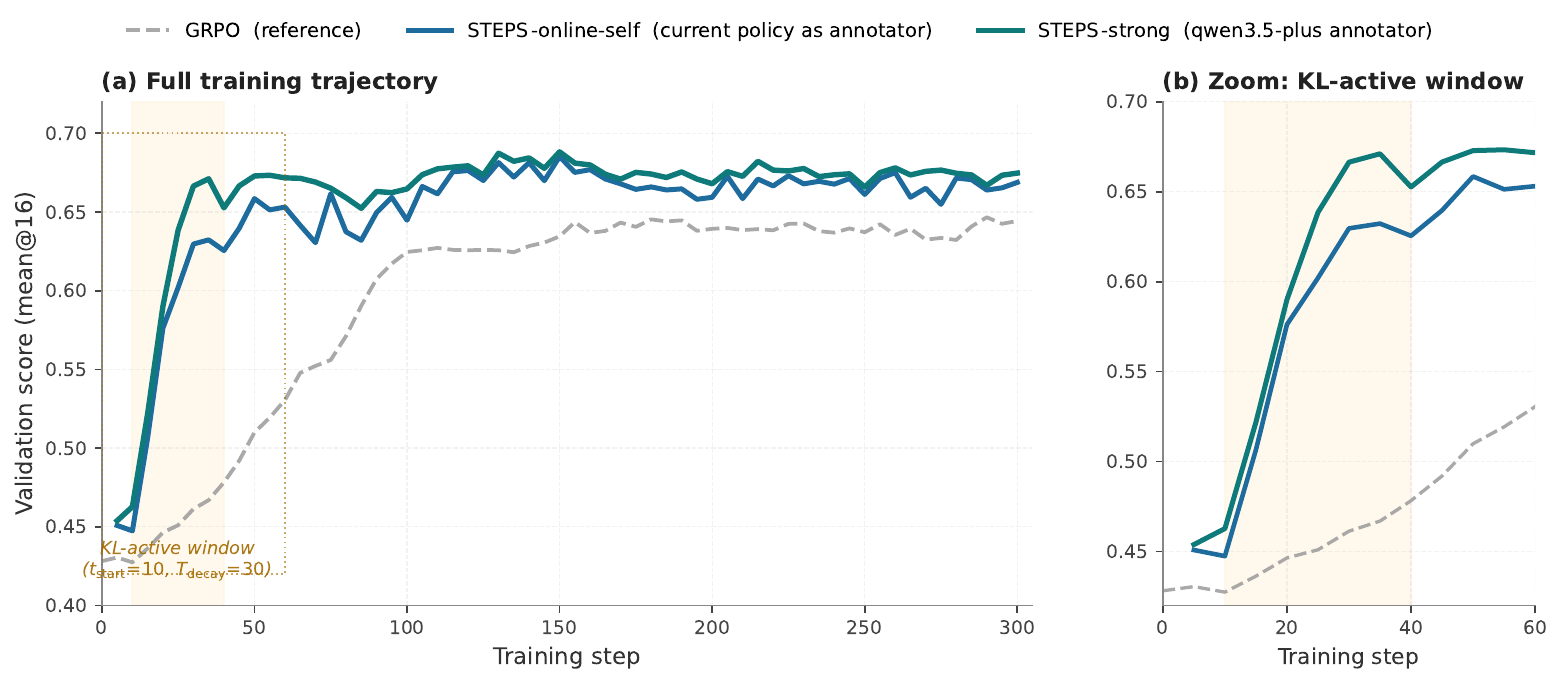}
  \caption{\textbf{Annotator ablation dynamics.} OpenThoughts held-out validation mean@16 for the annotators in \paperref{Table~\ref{tab:a8-annotator}}: (a) full training trajectory; (b) early-phase zoom. Shading marks distillation decay at steps 10--40; auxiliary computation ends after step 40. Curves use EMA smoothing ($\alpha=0.85$).}
  \label{fig:annot-dynamics}
\end{figure}

\subsection{Evaluation Protocol Sanity Checks}
\label{app:base-eval}

The Qwen3-8B base scores reported in the ``base'' row of \paperref{Tab.~\ref{tab:main}} are independently re-graded with an \emph{ensemble OR} grader (\texttt{verl/math\_dapo.py} $\vee$ HF \texttt{math-verify} $\vee$ AIME-int compare $\vee$ normalized-string compare); across all five eval cells the strict-vs-ensemble difference is $\le 0.03\%$, i.e.\ smaller than the bootstrap CI width on a problem set of this size. The headline strict-grader numbers are therefore not driven by grader choice. \paperref{Tab.~\ref{tab:main-ci}} reports cell-wise bootstrap 95\% confidence intervals over problems for the main held-out results. These intervals are intended for calibration of benchmark uncertainty; because AIME contains only 30 problems, we avoid treating individual AIME-column gaps as standalone significance claims.

\begin{table}[!htbp]
\centering\scriptsize
\setlength{\tabcolsep}{2pt}
\caption{\textbf{Cell-wise bootstrap 95\% confidence intervals} for \paperref{Tab.~\ref{tab:main}}. Each cell reports mean accuracy with a bootstrap interval over problem IDs. AVG is the unweighted mean over the five benchmark means; its shading follows \paperref{Tab.~\ref{tab:main}}.}
\label{tab:main-ci}
\begin{tabular}{lccccc c}
\toprule
Method & MATH-500 & AIME 24 & AIME 25 & AMC 23 & GPQA-D & AVG \\
\midrule
Qwen3-8B base (ref.) &
\shortstack{$96.80$\\$[95.48,98.00]$} &
\shortstack{$76.25$\\$[62.50,87.92]$} &
\shortstack{$67.50$\\$[54.17,80.42]$} &
\shortstack{$95.94$\\$[92.50,98.75]$} &
\shortstack{$58.27$\\$[52.21,64.20]$} &
\shortstack{$78.95$\\$[75.04,82.86]$} \\
\midrule
GRPO &
\shortstack{$97.30$\\$[96.03,98.45]$} &
\shortstack{$77.08$\\$[62.71,89.38]$} &
\shortstack{$68.96$\\$[54.58,82.08]$} &
\shortstack{$96.56$\\$[92.19,99.38]$} &
\shortstack{$53.85$\\$[48.04,59.66]$} &
\shortstack{$78.75$\\$[74.67,82.83]$} \\
SDPO &
\shortstack{$95.13$\\$[93.65,96.48]$} &
\shortstack{$52.50$\\$[38.33,66.67]$} &
\shortstack{$41.25$\\$[26.04,56.67]$} &
\shortstack{$91.41$\\$[84.69,96.25]$} &
\shortstack{$39.90$\\$[34.85,44.95]$} &
\cellcolor{degradationAvg}\shortstack{$64.04$\\$[59.58,68.50]$} \\
SRPO &
\shortstack{$96.48$\\$[95.13,97.65]$} &
\shortstack{$69.17$\\$[54.58,82.08]$} &
\shortstack{$54.17$\\$[39.58,67.50]$} &
\shortstack{$93.12$\\$[87.50,97.50]$} &
\shortstack{$37.31$\\$[32.32,42.42]$} &
\cellcolor{degradationAvg}\shortstack{$70.05$\\$[65.87,74.23]$} \\
RLSD &
\shortstack{$96.68$\\$[95.35,97.90]$} &
\shortstack{$75.42$\\$[61.67,87.08]$} &
\shortstack{$68.54$\\$[55.21,81.46]$} &
\shortstack{$97.03$\\$[93.12,99.38]$} &
\shortstack{$57.26$\\$[51.58,62.82]$} &
\shortstack{$78.99$\\$[74.96,82.98]$} \\
\midrule
\rowcolor{traceRow}
\method (RKL on $\mathcal{E}_y$) &
\shortstack{$97.48$\\$[96.18,98.55]$} &
\shortstack{$78.54$\\$[64.17,90.42]$} &
\shortstack{$70.42$\\$[56.25,83.33]$} &
\shortstack{$97.81$\\$[94.38,99.69]$} &
\shortstack{$56.44$\\$[50.06,62.69]$} &
\shortstack{$80.14$\\$[76.08,84.15]$} \\
\rowcolor{traceRow}
\method (FKL on $\mathcal{K}_y$) &
\shortstack{$97.83$\\$[96.65,98.83]$} &
\shortstack{$80.21$\\$[66.67,92.50]$} &
\shortstack{$73.54$\\$[59.58,86.25]$} &
\shortstack{$97.66$\\$[94.06,99.69]$} &
\shortstack{$58.33$\\$[52.53,64.14]$} &
\shortstack{$81.51$\\$[77.57,85.45]$} \\
\bottomrule
\end{tabular}
\end{table}

\subsection{SDPO/SRPO Re-Implementation Sweep}
\label{app:sdpo-sweep}

\paperref{Tab.~\ref{tab:sdpo-sweep}} lists the configurations we tried for the SDPO and SRPO baselines. The entries are our re-implementations of the published objectives under the training setup of this study. Each row gives the values explored and which value was selected for the headline rows of \paperref{Tab.~\ref{tab:main}}.

\paragraph{Evaluation setting.}
We evaluate the published SDPO and SRPO objectives under a shared long-horizon math RLVR setup with Qwen3 Think rollouts. The sweep covers two model scales and more than 12 configurations; \paperref{Figures~\ref{fig:val_reward} and~\ref{fig:three-symptom}} document the observed validation degradation, response shortening, and entropy rise.

\begin{table}[!htbp]
\centering\small
\setlength{\tabcolsep}{4pt}
\caption{SDPO/SRPO re-implementation sweep on Qwen3-8B + OpenThoughts math. ``Main setting'' lists the configuration used in \paperref{Table~\ref{tab:main}}, including symmetric training. Checkpoints maximize OpenThoughts validation mean@16 within each run; held-out benchmarks are not used for selection.}
\label{tab:sdpo-sweep}
\begin{tabular}{lll}
\toprule
\textbf{Factor} & \textbf{Values tested} & \textbf{Main setting} \\
\midrule
Learning rate          & $\{10^{-6},\ 5{\times}10^{-6},\ 10^{-5}\}$ & $1{\times}10^{-5}$ \\
Batch size             & $\{32,\ 256\}$                    & $32$ \\
EMA rate               & $\{0,\ 0.05\}$                    & SDPO $0$ / SRPO $0.05$ \\
Distillation top-$K$   & $\{20,\ 100\}$                    & $100$ \\
Divergence             & $\{\mathrm{FKL},\ \mathrm{RKL},\ \mathrm{JSD}\}$ & SDPO JSD ($\alpha{=}0.5$) / SRPO FKL \\
Thinking mode          & symmetric / asymmetric            & symmetric (student / teacher Think) \\
Per-vocab KL clip      & $\{0,\ 0.05\}$                    & $0.05$ \\
\bottomrule
\end{tabular}
\end{table}

\paperref{Figure~\ref{fig:three-symptom}} shows two failure symptoms across the sweep, length collapse and entropy rise, appearing within 150--200 steps; the third (validation collapse) is the headline observation of \S\ref{sec:exp:main}. The checkpoints reported in \paperref{Tab.~\ref{tab:main}} maximize validation mean@16 within their runs and fall in the early training window.
\subsection{Three-Symptom Failure Pattern of All-Token OPSD Baselines}
\label{app:three-symptom}

\paperref{Figure~\ref{fig:three-symptom}} shows the training diagnostics of all-token OPSD baselines and the SRPO inner mechanism.

\begin{figure}[!t]
  \centering
  \includegraphics[width=\linewidth]{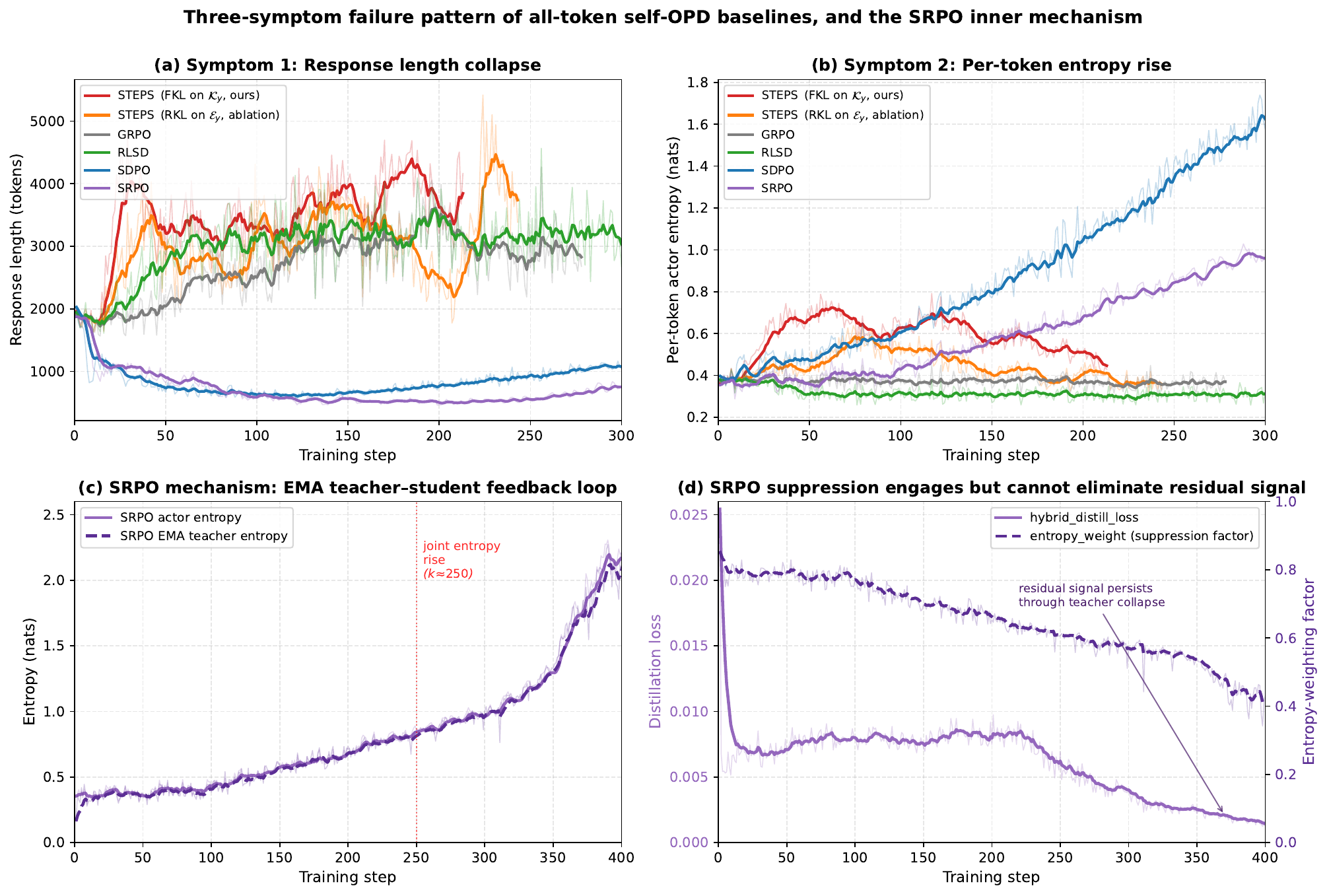}
  \caption{\textbf{All-token OPSD training diagnostics and SRPO dynamics.} (a) Response length; (b) per-token actor entropy. SDPO and SRPO produce shorter responses as entropy rises. (c) SRPO actor and EMA-teacher entropy co-vary, consistent with a possible feedback loop~\citep{kim2026why}, without establishing causality. (d) Entropy-aware suppression reduces its weight but leaves nonzero distillation loss during degradation. Curves use EMA smoothing ($\alpha=0.80$); faint traces show raw values.}
  \label{fig:three-symptom}
\end{figure}

\subsection{Thinking-Mode Ablations on Qwen Models}
\label{app:a9-asym}
\label{app:a10-nothink}

\paperref{Tables~\ref{tab:a9-asym} and~\ref{tab:a10-nothink}} report the Qwen3-8B ablations defined in \paperref{Table~\ref{tab:thinking-modes}}; decoding settings are centralized in \paperref{Appendix~\ref{app:decoding}}.

\paragraph{Asymmetric training.}
\method with FKL on key spans achieves a five-benchmark AVG of $81.51$ under the default symmetric configuration, compared with $75.79$ under asymmetric training when both are evaluated with thinking enabled. Although the asymmetric configuration has a lower AVG than the GRPO reference ($75.79$ vs.\ $78.75$), it retains a GPQA-Diamond improvement ($56.69$ vs.\ $53.85$).

\begin{table}[!ht]
\centering\small
\setlength{\tabcolsep}{4pt}
\caption{\textbf{Asymmetric-training ablation on Qwen3-8B.} Student NoThink, teacher Think; evaluation Think. The GRPO row is reused from \paperref{Table~\ref{tab:main}} as a reference and has no teacher.}
\label{tab:a9-asym}
\begin{tabular}{lccccc c}
\toprule
Method & MATH-500 & AIME 24 & AIME 25 & AMC 23 & GPQA-D & AVG \\
\midrule
Qwen3-8B base, Thinking ON (ref.) & 96.80 & 76.25 & 67.50 & 95.94 & 58.27 & 78.95 \\
\midrule
+ GRPO  & \textbf{97.30} & \textbf{77.08} & \textbf{68.96} & \textbf{96.56} & 53.85 & \textbf{78.75} \\
+ RLSD  & 95.85 & \underline{73.75} & \underline{65.42} & \underline{95.31} & \textbf{56.88} & \underline{77.44} \\
\rowcolor{traceRow}
+ \method (FKL on $\mathcal{K}_y$)   & 94.88 & 70.83 & 62.50 & 94.06 & \underline{56.69} & 75.79 \\
\rowcolor{traceRow}
+ \method (RKL on $\mathcal{E}_y$, ablation) & 93.95 & 68.33 & 60.42 & 93.12 & 56.44 & 74.45 \\
+ SRPO  & \underline{96.62} & 68.75 & 52.71 & 92.50 & 35.54 & 69.22 \\
+ SDPO  & 94.88 & 50.42 & 38.75 & 90.62 & 38.26 & 62.59 \\
\bottomrule
\end{tabular}
\end{table}

\FloatBarrier
\paragraph{Non-thinking evaluation.}

Reevaluating the same asymmetric-training checkpoints with the non-thinking protocol gives \method-FKL an AVG of $60.95\%$, exceeding the corresponding GRPO reference of $58.73\%$ by an absolute $2.22\%$. Relative to thinking-enabled evaluation, this protocol also shortens the generation budget and uses fewer samples on AIME / AMC; the comparison therefore reflects these joint decoding changes.

\begin{table}[!ht]
\centering\small
\setlength{\tabcolsep}{4pt}
\caption{\textbf{Non-thinking evaluation of the asymmetric-training configuration.} Student NoThink, teacher Think; evaluation NoThink on Qwen3-8B. GRPO is a teacher-free reference. Best per column is \textbf{bold}.}
\label{tab:a10-nothink}
\begin{tabular}{lccccc c}
\toprule
Method (NoThink eval) & MATH-500 & AIME 24 & AIME 25 & AMC 23 & GPQA-D & AVG \\
\midrule
Qwen3-8B base, Thinking OFF (ref.) & 84.20 & 27.08 & 19.17 & 68.12 & 48.11 & 49.34 \\
\midrule
+ GRPO  & 91.47 & \textbf{43.75} & 29.58 & 80.00 & \textbf{48.86} & 58.73 \\
+ RLSD  & 90.88 & \underline{43.33} & \underline{33.75} & \underline{82.50} & 45.22 & \underline{59.14} \\
+ SDPO  & 78.85 & 15.83 & 11.67 & 58.75 & 28.66 & 38.75 \\
+ SRPO  & 62.52 &  8.33 &  5.00 & 40.94 & 36.49 & 30.66 \\
\rowcolor{traceRow}
+ \method (RKL on $\mathcal{E}_y$, ablation) & 86.88 & 26.67 & 25.83 & 80.00 & \underline{48.36} & 53.55 \\
\rowcolor{traceRow}
\textbf{+ \method (FKL on $\mathcal{K}_y$)}  & \textbf{92.25} & 42.08 & \textbf{36.25} & \textbf{87.81} & 46.34 & \textbf{60.95} \\
\bottomrule
\end{tabular}
\end{table}

\FloatBarrier
\subsection{Negative Results}
\paragraph{Generic rubrics and free-form critique (negative).}
A pre-\method exploration phase tested query-level rubric generators (Rubric-ARM-style~\citep{xx2026rubricarm}) and free-form LLM critique as alternative teacher conditioning. Both underperformed span-localized routing in our setup while incurring substantially higher annotator cost. The negative finding motivated the per-rollout type-label design adopted in \method.

\paragraph{All-token forward KL plus decay (negative).}
We further tested whether the same FKL signal applied to \emph{all} response tokens, with the same decay schedule, would suffice. Under matched warm-up and decay, full-token FKL still degraded GPQA-Diamond OOD relative to the span-localized variant by step $100$, indicating that span localization (not just decay) is necessary.

\section{Additional Controlled Experiments and Accounting}
\label{app:controlled-details}

\subsection{Coverage and Distillation Duration}
\label{app:coverage-duration}

\paperref{Table~\ref{tab:coverage-duration}} reports the coverage sweep on Qwen3-1.7B. Scores are the unweighted five-benchmark AVG used for the Qwen3 main results; GRPO is $57.73$. The decayed runs use the same decay-to-GRPO schedule, and the constant-weight run uses $\lambda=0.5$. Among the four measured coverage settings, $\alpha=0.25$ performs best. These single-seed results characterize sensitivity in this configuration.

\begin{table}[!htbp]
\centering\small
\setlength{\tabcolsep}{12pt}
\caption{\textbf{Coverage and duration sweep on Qwen3-1.7B.} The constant-KL comparison uses $\alpha=0.25$. The $\alpha=1.00$ row applies KL to all tokens.}
\label{tab:coverage-duration}
\begin{tabular}{ccc}
\toprule
Coverage $\alpha$ & Decayed KL & Constant KL ($\lambda=0.5$) \\
\midrule
$0.10$ & $58.20$ & --- \\
$0.25$ & $\mathbf{60.16}$ & $54.85$ \\
$0.50$ & $56.90$ & --- \\
$1.00$ & $56.68$ & --- \\
\bottomrule
\end{tabular}
\end{table}

The constant-weight comparison evaluates the complete hand-off design. In Eq.~\eqref{eq:loss}, $\rho_k=1-\lambda_k/w_0$ restores GRPO on selected spans as KL decays. Holding $\lambda=w_0=0.5$ therefore also holds $\rho_k=0$, rather than restoring the span-token GRPO term. The drop from $60.16$ to $54.85$ cannot be attributed solely to persistent KL independently of that coupling. Likewise, the all-token result $56.68$ uses matched decay and differs from GRPO by $-1.05\%$.

The class-by-direction experiment in \paperref{Figure~\ref{fig:design-ablations}(b)} uses the same seed, annotator, coverage, schedule, and checkpoint-selection rule across the four cells. It separates the span-class and divergence-direction changes that are coupled in the original two-corner comparison. The values are single-seed evaluation results, not averages over independent training seeds.

\subsection{Paired Evaluation Uncertainty}
\label{app:paired-uncertainty}

\paperref{Table~\ref{tab:paired-uncertainty}} reports paired bootstrap comparisons for Qwen3-8B. Evaluation outcomes are paired by problem ID between methods, using $10{,}000$ resamples. The target statistic is the difference in AVG, the unweighted mean of the five benchmark accuracies. Pairing preserves the correspondence between methods on each evaluated problem. \paperref{Table~\ref{tab:main-ci}} reports cell-wise intervals for individual benchmark accuracies, while the paired intervals quantify uncertainty in method differences.

\begin{table}[!htbp]
\centering\small
\setlength{\tabcolsep}{10pt}
\caption{\textbf{Paired evaluation differences on Qwen3-8B.} Absolute accuracy differences and their confidence-interval endpoints are expressed in \%. Strong uses \texttt{qwen3.5-plus}; online-self reuses the current policy as annotator.}
\label{tab:paired-uncertainty}
\begin{tabular}{lcc}
\toprule
Comparison & $\Delta$ AVG & Paired 95\% CI \\
\midrule
\method-strong $-$ GRPO & $+2.76$ & $[0.41,\ 4.91]$ \\
\method-strong $-$ RLSD & $+2.52$ & $[0.17,\ 3.67]$ \\
\rowcolor{onlineSelfRow}
\method-online-self $-$ GRPO & $+1.90$ & $[-0.25,\ 3.75]$ \\
\bottomrule
\end{tabular}
\end{table}

The online-self interval includes zero, so we report its observed mean improvement without a significance claim. These intervals quantify uncertainty from evaluation problems conditional on the trained checkpoints. They do not measure variability across training seeds.

\subsection{Online-Self Timing Breakdown}
\label{app:online-timing}

\paperref{Table~\ref{tab:online-timing}} reports the online-self timing components at three model sizes. Annotation and the teacher forward are additional operations during the early distillation phase, confined to the first 40 training steps; when $\lambda_k=0$, Algorithm~\ref{alg:trace} skips both operations. The 4B column is a timing measurement only and is not an additional model in the effectiveness evaluation.

\begin{table}[!htbp]
\centering\small
\setlength{\tabcolsep}{10pt}
\caption{\textbf{Online-self timing.} Component times are seconds per KL-active step. The last row reports overhead amortized over a 300-step run. Accelerator counts differ across columns, so absolute seconds cannot be compared directly across model sizes.}
\label{tab:online-timing}
\begin{tabular}{lrrr}
\toprule
Component & 1.7B & 4B & 8B \\
\midrule
Shared rollout / forward / actor & $124.4$ & $117.3$ & $131.2$ \\
Annotator generation & $15.8$ & $16.7$ & $17.2$ \\
Teacher forward & $6.7$ & $7.4$ & $9.4$ \\
Auxiliary total & $22.5$ & $24.1$ & $26.6$ \\
\midrule
Overhead across 300 steps & $2.4\%$ & $2.7\%$ & $2.7\%$ \\
\bottomrule
\end{tabular}
\end{table}

Amortized overhead is computed over 300 training steps, using a 40-step allowance for auxiliary work:
\begin{equation}
\mathrm{overhead}(\%)
=100\,\frac{40\,t_{\mathrm{aux}}}{300\,t_{\mathrm{shared}}}.
\label{eq:timing-accounting}
\end{equation}
This yields $2.41\%$, $2.74\%$, and $2.70\%$ for 1.7B, 4B, and 8B, respectively, before rounding. The denominator is shared training work over all 300 steps, not the duration of a KL-active step. The 40-step allowance covers the initial phase; the scheduled nonzero-KL window is specified separately in Eq.~\eqref{eq:lambda-schedule}.

This timing breakdown is an overhead accounting, not a compute-matched comparison. It applies to online-self annotation and does not estimate the cost or latency of the strong API annotator.

\end{document}